\documentclass{article} 
\usepackage{iclr2021_conference,times}

\usepackage{amsmath,amsfonts,bm}

\def\eqref#1{equation~\ref{#1}}

\def\1{\bm{1}}

\def\va{{\bm{a}}}

\def\vh{{\bm{h}}}

\def\mA{{\bm{A}}}

\def\mH{{\bm{H}}}

\def\mW{{\bm{W}}}

\DeclareMathAlphabet{\mathsfit}{\encodingdefault}{\sfdefault}{m}{sl}
\SetMathAlphabet{\mathsfit}{bold}{\encodingdefault}{\sfdefault}{bx}{n}

\newcommand{\R}{\mathbb{R}}

\usepackage{hyperref}
\usepackage{url}

\iclrfinalcopy 

\usepackage[utf8]{inputenc} 
\usepackage[T1]{fontenc}    
\usepackage{hyperref}       
\usepackage{url}            
\usepackage{booktabs}       
\usepackage{amsfonts}       
\usepackage{nicefrac}       
\usepackage{microtype}      
\usepackage{color} 
\usepackage{graphicx}
\usepackage{amsmath}
\usepackage{amssymb}
\usepackage{bm}
\usepackage[ruled,vlined, noend]{algorithm2e}

\usepackage{hhline}
\usepackage{multirow}

\DontPrintSemicolon 

\newtheorem{definition}{Definition}

\newcommand{\mname}{\text{FastGAT}\xspace }
\newcommand{\Lrw}{\bm{L_{\mathrm{rw}}}}
\newcommand{\Lrws}{\bm{L_{\mathrm{rw,s}}}}
\newcommand{\Lsym}{\bm{L_{\mathrm{norm}}}}
\newcommand{\Lsyms}{\bm{L_{\mathrm{norm,s}}}}
\newcommand{\bmx}{\bm{x}}
\newcommand{\bmD}{\bm{D}}
\newcommand{\bmDs}{\bm{D_s}}
\newcommand{\bmA}{\bm{A}}
\newcommand{\newconv}{\widehat{\star}}
\newcommand{\rpm}{\raisebox{.2ex}{$\scriptstyle\pm$}}

 \usepackage{enumerate}
\usepackage{enumitem}
 
 \DeclareMathOperator{\es}{EdgeSample}

\title{Fast Graph Attention Networks Using Effective Resistance Based Graph Sparsification}

\author{
Rakshith S Srinivasa \\
  Georgia Institute of Technology\\
  Atlanta, GA  \\
\And Cao Xiao\\
  IQVIA \\
  Cambridge, MA \And Lucas  Glass \\IQVIA \\Cambridge, MA  \And   Justin Romberg\\
  Georgia Institute of Technology\\
  Atlanta, GA  \\
\\
 \And Jimeng Sun \\
 University of Illinois, Urbana-Champagne\\
 Urbana, IL 
}

\newcommand{\norm}[1]{\left \|#1 \right \|}

 \newtheorem{lemma}{Lemma}
 \newtheorem{theorem}{Theorem}
 \newtheorem{proposition}{Proposition}
 \newenvironment{proof}{\noindent {\bf Proof} }{\endprf\par}
 \def \endprf{\hfill {\vrule height6pt width6pt depth0pt}\medskip}

\newcommand{\wa}{\widetilde{A}}

\graphicspath{ {Figs/} }

\newcommand\blfootnote[1]{%
  \begingroup
  \renewcommand\thefootnote{}\footnote{#1}%
  \addtocounter{footnote}{-1}%
  \endgroup
}

\begin{document}

\maketitle

\begin{abstract}

    The attention mechanism has demonstrated superior performance for inference over nodes in graph neural networks (GNNs), however, they result in a high computational burden during both training and inference. 
  We propose \mname, a method to make attention based GNNs lightweight by using spectral sparsification to generate an optimal pruning of the input graph. This results in a per-epoch time that is almost linear in the number of graph nodes as opposed to quadratic. We theoretically prove that spectral sparsification preserves the features computed by the GAT model, thereby justifying our \mname algorithm. We experimentally evaluate \mname on several large real world graph datasets for node classification tasks under both inductive and transductive settings. \mname can dramatically reduce (up to \textbf{10x}) the computational time and memory requirements, allowing the usage of attention based GNNs on large  graphs.  
\end{abstract}

\section{Introduction}

Graphs are efficient representations of pairwise relations, with many real-world applications including product co-purchasing network ~(\citep{10.1145/2783258.2783381}), co-author network~(\citep{DBLP:journals/corr/HamiltonYL17}), etc.
Graph neural networks (GNN) have become popular as a tool for inference from graph based data. By leveraging the geometric structure of the graph, GNNs learn improved representations of the graph nodes and edges that can lead to better performance in various inference tasks ~(\citep{kipf2016semi,hamilton2017inductive,velickovic2018graph}). More recently, the attention mechanism has demonstrated superior performance for inference over nodes in GNNs~(\citep{velickovic2018graph, xinyi2018capsule, thekumparampil2018attention, lee2020heterogeneous, bianchi2019hierarchical, knyazev2019understanding}). However,  attention based GNNs suffer from huge computational cost. This may hinder the applicability of the attention mechanism to large  graphs. 

GNNs generally rely on graph convolution operations. For a graph $G$ with $N$ nodes, graph convolution with a kernel $\bm{g_w}: \R \rightarrow \R$ is defined as
\begin{equation}
   \bm{g_w} \star \bm{h} = \bm{U}\bm{g_w(\Lambda)}\bm{U}^\top \bm{h}
    \label{eq:gconv} 
\end{equation} 
where $\bm{U}$ is the matrix of eigenvectors and $\bm{\Lambda}$ is the diagonal matrix of the eigenvalues of the normalized graph Laplacian matrix defined as 
\begin{equation}\Lsym =\bm{I} - \bm{D^{-1/2} A D^{-1/2}},
\label{eq:Lsymdef}
\end{equation}

with $\bm{D}$ and $\bm{A}$ being the degree matrix and the adjacency matrix of the graph, and $\bm{g_w}$ is applied elementwise.
Since computing $\bm{U}$ and $\bm{\Lambda}$ can be very expensive ($O(N^3)$), most GNNs use an approximation of the graph convolution operator. For example, in graph convolution networks (GCN) \citep{kipf2016semi}, node features are updated by computing averages as a first order approximation of Eq.\eqref{eq:gconv} over the neighbors of the nodes. A single neural network layer is defined as:
\begin{equation}
    \bm{H}^{(l+1)}_{\mathrm{GCN}} = \sigma \left ( \widetilde{\bmD}^{-1/2} \widetilde{\bmA} \widetilde{\bmD}^{-1/2} \bm{H^{(l)}} \bm{W^{(l)}} \right ),
    \label{eq:GCN}
\end{equation} where $\bm{H^{(l)}}$ and $\bm{W^{(l)}}$ are the activations and the weight matrix at the $l$th layer respectively and $\widetilde{\bmA} = \bmA + \bm{I}$ and $\widetilde{\bmD}$ is the degree matrix of $\widetilde{\bmA}$. 


Attention based GNNs add another layer of complexity: they compute pairwise attention coefficients between all connected nodes. This process can significantly increase the computational burden, especially on large graphs. Approaches to speed up GNNs were proposed in \citep{chen2018fastgcn,hamilton2017inductive}. However, these sampling 
and aggregation based methods were designed for simple GCNs and are not applicable to attention based GNNs. There has also been works in inducing sparsity in attention based GNNs \citep{ye2019sparse, zheng2020robust}, but they focus on addressing potential overfitting of attention based models rather than scalability.

In this paper, we propose {\it Fast} {\it G}raph {\it A}ttention ne{\it T}work (\mname), an edge-sampling based method that leverages effective resistances of edges to make attention based GNNs lightweight.  The effective resistance measures importance of the edges  in terms of preserving the graph connectivity. \mname uses this measure to  prune the input graph and generate a randomized subgraph with far fewer edges. Such a procedure preserves the spectral features of a graph, hence retaining the information that the attention based GNNs  need. At the same time, the graph is amenable to more complex but computationally intensive models such as attention GNNs. With the sampled subgraph as their inputs, the attention based GNNs  enjoy  much smaller computational complexity. Note that \mname is applicable to all attention based GNNs. In this paper, we mostly focus on the Graph Attention NeTwork model (GAT) proposed by \citep{velickovic2018graph}. However we also show \mname is generalizable  to two other attention based GNNs, namely the cosine similarity based approach \citep{thekumparampil2018attention} and Gated Attention Networks \citep{zhang2018gaan}.

We note that Graph Attention Networks can be re-interpreted as convolution based GNNs. We show this explicitly in the Appendix. Based on this re-interpretation, we theoretically prove that spectral sparsification preserves the feature representations computed by the GAT model. We believe this interpretation also opens up interesting connections between sparsifying state transition matrices of random walks and speeding up computations in GNNs. 

The contributions of our paper are as outlined below:
\begin{itemize}[leftmargin=*]
    \item We propose \mname, a method that uses effective resistance based spectral graph sparsification to accelerate attention GNNs in both inductive and transductive learning tasks. The rapid subsampling and the spectrum preserving property of \mname help attention GNNs retain their accuracy advantages and become computationally light.
    
    \item  We provide a theoretical justification for using spectral sparsification in the context of attention based GNNs by  proving that spectral sparsification preserves the features  computed by GNNs.
    
      \item  \mname outperforms state-of-the-art algorithms across a variety of datasets under both transductive and inductive settings in terms of computation, achieving a speedup of up to \textbf{10x} in training and inference time. 
      On larger datasets such as Reddit, the standard GAT model runs out of memory, whereas \mname achieves an F1 score 0.93 with  7.73 second per epoch time in training.
    
      \item Further, \mname is generalizable to other attention based GNNs such as the cosine similarity based attention \citep{thekumparampil2018attention} and the Gated Attention Network \citep{zhang2018gaan}.
\end{itemize}

\section{Related work}

\noindent{\bf Accelerating graph based inference}
has drawn increasing interest. Two methods proposed in \citep{chen2018fastgcn} (FastGCN) and \citep{huang2018adaptive} speed up GCNs by using importance sampling to sample a subset of nodes per layer during training. 
Similarly, GraphSAGE~ \citep{hamilton2017inductive} also proposes an edge sampling 
and aggregation based method for inductive learning based tasks. All of the above works use simple aggregation and target simple GCNs, while our work focus on more recent attention based GNNs such as \citep{velickovic2018graph}. We are able to take advantage of the attention mechanism, while still being computationally efficient. 

\noindent{\bf Graph sparsification} aims to approximate a given graph by a graph with fewer edges for efficient computation.  Depending on  final goals, there are cut-sparsifiers ~(\citep{cut-sparse}), pair-wise distance preserving sparsifiers~(\citep{distance-sparse}) and spectral sparsifiers ~(\citep{spielman2004nearly, spielman11graph}) , among others~(\citep{zhao2015gsparsify, calandriello2018improved, hubler2008metropolis, eden2018provable, Lapsparse}).  In this work, we use spectral sparsification to choose a randomized subgraph while preserving spectral properties. Apart form providing the strongest guarantees in preserving graph structure (\citep{chu2018graph}), they align well with GNNs  due to their connection to spectral graph convolutions.

\noindent{\bf Graph sparsification on neural networks}
have been studied recently (\citep{ye2019sparse, zheng2020robust, Ioannidis2020Pruned,louizos2017learning}). However, their main goal is to alleviate overfitting in GNNs not reducing the training time. They still require learning attention coefficients and binary gate values for all edges in the graph, hence not leading to any computational or memory benefit. In contrast, \mname uses a fast subsampling procedure, thus resulting in a drastic improvement in training and inference time. It is also highly stable in terms of training and inference. 

\section{\mname: Accelerating graph attention networks via edge sampling }\label{sec:model}
\subsection{The \mname algorithm}

Let $\bm{G}(\mathcal{E},\mathcal{V})$ be a graph with $N$ nodes and $M$ edges. An attention based GNN computes attention coefficients $\alpha_{i,j}$ for every pair of connected nodes $i, j \in \mathcal{V}$ in every layer $\ell$. The $\alpha_{i,j}$'s are then used as averaging weights to compute the layer-wise feature updates. In the original GAT formulation, the attention coefficients are
\begin{equation}
    \alpha_{ij} = \frac{\exp \left (\text{LeakyReLU} (\bm{a}^\top [ \bm{W}\bm{h_i} || \bm{W}\bm{h_j} ] ) \right ) } {\sum_{j \in \mathcal{N}_{i}} \exp \left (\text{LeakyReLU} (\bm{a}^\top [ \bm{W}\bm{h_i} || \bm{W}\bm{h_j} ] ) \right ) },
    \label{eq:attention}
\end{equation} where $\bm{h_i}$'s are the input node features to the layer, $\bm{W}$ and $\bm{a}$ are linear mappings that are learnt, $\mathcal{N}_{i}$ denotes the set of neighbors of node $i$, and $||$ denotes concatenation. With the $\alpha_{ij}$'s as defined above, the node-$i$ output embedding  of a GAT layer is 
\begin{equation}
    \bm{h_i}' = \sigma \left ( \sum_{j \in \mathcal{N}_{i}} \alpha_{ij}\bm{W} \bm{h_j} \right ).
    \label{eq:GAT}
\end{equation} For multi-head attention, the coefficients are computed independently in each attention head with head-dependent matrices $\mW$ and attention vector $\va$. Note that the computational burden in GATs arises directly from computing the $\alpha_{i,j}$'s in \textit{every layer, every attention head and every forward pass} during training. 

\noindent{\bf Goal:}  Our objective is to achieve performance equivalent to that of full graph attention networks (GAT), but with only a fraction of the original computational complexity.
This computational saving is achieved by reducing the number of attention computations. 

\noindent{\bf Idea:} We propose to use {\bf edge-sampling functions} that sparsify graphs by removing nonessential edges. This leads to direct reduction in the number of attention coefficients to be computed, hence reducing the burden. Choosing the sampling function  is crucial for retaining the graph connectivity.

Let $\es(E,\mA,q)$ denote a \textit{randomized sampling function} that, given an edge set $E$, adjacency matrix $\mA$ and a number of edges to be sampled $q$, returns a subset of the original edge set $E_s \subset E$ with $|E_s| = q$. Our algorithm then uses this function to sparsify the graph in every layer and attention head. Following this, the attention coefficients are computed \textit{only for the remaining edges}. A more detailed description is given in Algorithm \ref{algo:FastGAT}. In every layer and attention head, a randomized subgraph with $q \ll M$ edges is generated and the attention coeffients are computed only for this subset of edges.
We use a specialized distribution that depends on the contribution of each edge to the graph connectivity. We provide further details in Section \ref{subsec:Reff}. 

\begin{algorithm}[h!]
\SetAlgoLined
\textbf{Input:} Graph $G(\mathcal{V}, \mathcal{E})$, Num. layers = $L$, Num. Attention heads $K^{(\ell)}$, $\ell = 1, \ \cdots, \ L$ \\
\hspace{0.4in} Initial Weight matrices $\mW^{(\ell)}$, Non-linearity $\sigma$, Feature matrix $H \in \R^{N\times D}$ \\
\hspace{0.4in} Randomized edge sampling function $\es(\cdot)$, Attention function $\phi(\cdot)$\\
\hspace{0.4in} Num. edges sampled $q$\\
\For{each layer $\ell$}{
\For{each attention head $k \in \{1, 2, \ \cdots, \ K^{(\ell)}\}$}{
Sample a graph $g^{(\ell)}_k = \es(E,\mA, q)$\;
Compute $\Gamma^{(\ell)}_k \in \R^{N \times N}, \ \Gamma^{(\ell)}_k(i,j) = \phi_{\theta_{k}}(\vh^{(\ell)}_i, \vh^{(\ell)}_j)$ if nodes $i,j$ are connected in $g^{(\ell)}_k$\;
Compute $\mH^{(\ell+1)}_k = \sigma \left (\Gamma^{(\ell)}_k \mH^{(\ell)}_k \mW^{(\ell)} \right )$\;
}
$\mH^{(\ell+1)} = \underset{k}{||}\mH^{(\ell)}_k$ \tcp{Concatenate the output of attention heads}
}
Compute loss and update $\mW$'s \tcp{gradient based weight update}
\caption{The \mname Algorithm}
\label{algo:FastGAT}
\end{algorithm}
Although in Algorithm \ref{algo:FastGAT} we sample a new subgraph in every layer and attention head, variations of this algorithm maybe used, depending on the cost the sampling function itself. Two simpler variations of \mname include: i) \mname-const, where the subgraph $g$ is kept constant in all the layers and attention heads and ii) \mname-layer, where the subgraph is differnet in each layer, but the same across all the attention heads within a layer.

\subsection{Sampling graph edges using effective resistances}
\label{subsec:Reff}
We use a particular edge sampling function $\es(\cdot)$ that is motivated by the field of spectral graph sparsification.
Let $\bm{L}$ represent the graph Laplacian (defined as $\bm{L} = \bm{D}-\bm{A}$ where $\bm{D}$ is the degree matrix), $\lambda_i(\bm{L})$ denote the $i$th eigenvalue of $\bm{L}$ and let $\bm{A}^\dagger$ denote the Moore-Penrose inverse of a matrix. 

Motivated by the fact that GNNs are approximations of spectral graph convolutions (defined in \eqref{eq:gconv}), we aim to preserve the spectrum (or eigenstructure) of the graph. 
Formally, let $\bm{L_G}$ and $\bm{L_H}$ be the Laplacian matrices of the original graph $\bm{G}$ and the sparsified graph $\bm{H}$. Then,  spectral graph sparsification ensures that the spectral content of $\bm{H}$ is similar to that of $\bm{G}$:
\begin{equation}
    (1-\epsilon) \lambda_i (\bm{L_G}) \leq \lambda_i (\bm{L_H}) \leq (1+\epsilon) \lambda_i (\bm{L_G}), \ \forall i
    \label{eq:eps-guaranty}
\end{equation} where $\epsilon$ is any desired threshold. ~\citep{spielman11graph} showed how to achieve this by using a distribution proportional to the effective resistances of the edges 


\begin{definition}[Effective Resistance] \citep{spielman11graph}
The effective resistance between any two nodes of a graph can be defined as the potential difference induced across the two nodes, when a unit current is induced at one node and extracted from the other node. Mathematically, it is defined as below.
\begin{equation*}
R_e(u,v) = \bm{b_e}^\top \bm{L}^{\dagger} \bm{b_e},
\end{equation*} 
where $\bm{b_e} = \bm{\chi_u} - \bm{\chi_v}$ ($\bm{\chi_l}$ is a standard basis vector with $1$ in the $l$th position) and $\bm{L}^\dagger$ is the pseudo-inverse of the graph Laplacian matrix.
\label{def:reff}
\end{definition}
The effective resistance measures the importance of an edge to the graph structure. For example, the removal of an edge with high effective resistance can harm the graph connectivity. The particular function $\es$ we use in \mname is described in Algorithm \ref{algo:main}.

\begin{algorithm}[h!]
\SetAlgoLined
 \textbf{Input}: Graph $\bm{G}(\mathcal{E}_G, \mathcal{V})$, $w_e$ is the edge weight for $e \in \mathcal{E}$, $ 0< \epsilon < 1$ \;
For each edge $e(u,v)$, compute $R_e(u,v)$ using fast algorithm in \citep{spielman11graph} \\
Set $q = \max(M,\mathrm{int}( 0.16 N \log N/\epsilon^2 ))$, $\bm{H}$ = Graph($\mathcal{E}_H = \mathrm{Null}$, $\mathcal{V}$)\\
\For{$i\leftarrow 1$ \KwTo $q$}{ Sample an edge $e_i$ from the distribution $p_e$ proportional to $w_e R_e$ \;
\If{$e_i \in \mathcal{E}_H$}{Add $w_e/qp_e$ to its weight
\tcp*{Increase  weight of an existing edge}}
\lElse{Add $e_i$ to $ \mathcal{E}_H$ with weight $w_e/ qp_e$
\tcp*{Add the edge for the first time}}} 
$\bm{H}$ = Graph$(\mathcal{E}_H, \mathcal{V})$\\ \caption{Effective resistance based $\es$ function \citep{spielman11graph}}
 \label{algo:main}
\end{algorithm}
The effective-resistance based edge-samplng function is described in Algorithm \ref{algo:main}. For a desired value of $\epsilon$, the algorithm sampled $q = O(N\log N/\epsilon^2)$ number of edges such that \eqref{eq:eps-guaranty} is satisfied.

\noindent\textbf{Choosing $\epsilon$}.  As shown in Algorithm.~\ref{algo:main}, it requires setting a pruning parameter $\epsilon$, which determines the quality of approximation after sparsification and also determines the number of edges retained $q$. The choice of $\epsilon$ is a design parameter at the discretion of the user. 
To remove the burden of choosing $\epsilon$, we also provide an \textbf{adaptive algorithm} in Section \ref{subsec: epsilon_effect} in the appendix.

\noindent\textbf{Complexity}. The sample complexity $q$ in Algorithm.~\ref{algo:main} directly determines the final complexity. If $q = O(N \log N / \epsilon^2)$, then the spectral approximation in \eqref{eq:eps-guaranty} can be achieved \citep{spielman11graph}. Note that this results in a number of edges that is almost linear in the number of nodes, as compared to quadratic as in the case of dense graphs. 
The complexity of computing $R_e$ for all edges is $O(M \log N)$  time, where $M$ is the number of edges~\citep{spielman11graph}. While we describe the algorithm in detail in the appendix (Section \ref{subsec:compute_reff}) , it uses a combination of fast solutions to Laplacian based linear systems and the Johnson-Lindenstrauss Lemma \footnote{https://en.wikipedia.org/wiki/Johnson-Lindenstrauss\_lemma}.
This is almost linear in the number of edges, and hence much smaller than the complexity of computing attention coefficients in every layer and forward pass of GNNs. Another important point is that the computation of $R_e$'s is a one-time cost. Unlike graph attention coefficients, we do not need to recompute the effective resistances in every training iteration. Hence, once sparsified, the same graph can be used in all subsequent experiments. Further, since each edge is sampled independently, the edge sampling process itself can be parallelized.

\section{Theoretical Analysis of \mname}
\label{sec:theory}
 In this section we provide the theoretical analysis of \mname.
 Although we used the sampling strategy provided in \citep{spielman11graph}, their work address the preservation of only the eigenvalues of $\bm{L}$. However, we are interested in the following question: {\it Can preserving the spectral structure of the graph  lead to good performance under the GAT model?} To answer this question, we give an upper bound on the error between the feature updates computed by a single layer of the GAT model using the full graph and a sparsified graph produced by \mname.

Spectral sparsification preserves the spectrum of the underlying graph. This then hints that neural network computations that utilize spectral convolutions can be approximated by using sparser graphs. We first show that this is true in a layer-wise sense for the GCN \citep{kipf2016semi} model and then show a similar result for the GAT model as well. Below, we use ReLU to denote the standard Rectified Linear Unit and ELU to denote the  Exponential Linear Unit.

\begin{theorem} At any layer $l$  of a GCN model with input features $\bm{H}^{(l)} \in \R^{N \times D }$, weight matrix $\bm{W}^{(l)} \in \R^{D \times F}$, if the element-wise non-linearity function $\sigma$ is either the ReLU or the ELU function, the features $\bm{\widehat{H}_f}$ and $\bm{\widehat{H}_s}$ computed using \eqref{eq:GCN} with the full and a layer dependent spectrally sparsified graph obey 
\begin{equation}
    \norm{\bm{\widehat{H}_f} - \bm{\widehat{H}}_s}_F \leq 4\epsilon \norm{\bm{\Lsym}} \norm{\bm{HW}}_F.
\end{equation} where $\Lsym$ is as defined in \eqref{eq:Lsymdef} and $\norm{\cdot}$ denotes the spectral norm.
\label{thm:GCN}
\end{theorem}

In our next result, we show a similar upper bound on the features computed with the full and the sparsified graphs using the GAT model.

\begin{theorem}
At any layer $l$ of GAT  with input features $\mH \in \R^{N \times D}$, weight matrix $\bm{W}^{(l)} \in \R^{D \times F}$ and $\alpha_ij$'s be the attention coefficients in that layer. Let the non-linearity used by either ReLU or the ELU functon. Then, the features $\bm{\widehat{H}_f}$ and $\bm{\widehat{H}_s}$ computed using \eqref{eq:GAT} with the full and a layer dependent spectrally sparsified graph obey 
\begin{equation}
    \norm{\bm{\widehat{H}_f} - \bm{\widehat{H}_s}}_F \leq 6 \epsilon \norm{\Lsym}\norm{\bm{HW}}_F
\end{equation} 
\label{thm:GAT} where $\norm{\cdot}$ denotes the spectral norm of the matrix.
\end{theorem}

Theorem \ref{thm:GAT} shows that our proposed layer-wise spectral sparsification leads to good approximations of latent embedding $\bm{\widehat{H}}$ for GAT model as well. The guarantees given above assume a layer-wise sparsification that is updated based on the attention coefficients. To circumvent the associated computational burden, we use the simpler versions such as \ref{algo:FastGAT}-const and always use the original weight matrix to sparsify the graph in each layer. In the experiment section, we show that such a relaxation by a one-time spectral sparsification does not lead to any degradation in performance.

\textbf{Approximation of weight matrices}.
Theorems \ref{thm:GCN} and \ref{thm:GAT} provide an upper bound on the feature updates obtained using the full and sparsified graphs. In practice, we  observe an even stronger notion of approximation between GAT and \mname: the weight matrices of the two models post training are  good approximations of each other. We report this observation in Section.~\ref{subsec:approx_weight} in the appendix. We show that the error between the learned matrices is small and proportional to the value of $\epsilon$ itself.

\section{Experiments}
\label{sec:exp}







\noindent\textbf{Datasets}
We evaluated \mname on large  graph datasets using  semi-supervised node classification tasks. This is a standard task to evaluate GNNs, as done in \citep{velickovic2018graph, hamilton2017inductive,kipf2016semi}. Datasets are sourced from the DGLGraph library \citep{DGL}. Their statistics  are provided in Table \ref{tab:data_descr}. We evaluate on both transductive and inductive tasks. The PPI dataset serves as a standard benchmark for inductive classification and the rest of the datasets for transductive classification. Further details about the datasets including details about train/validaton/ test split are given in the appendix (Section \ref{subsec:data}). We also evaluated on smaller datasets including Cora, citeseer and Pubmed, but present their results in the appendix (Section \ref{subsec:small_data}).

\begin{table}[h!]
\vskip -6pt
    \centering
       \caption{Dataset Statistics}
    \label{tab:data_descr}
    \resizebox{\textwidth}{!}{
    \begin{tabular}{l|ccccccc}
        \toprule
         Dataset & Reddit & Coauth-Phy & Github  &  Amaz.Comp & Coauth-cs . & Amaz.Photos & PPI (Inductive task)   \\
        \midrule
         Nodes & 232,965 & 34,493& 37,700 & 13,752 & 18,333  &  7,650 & 56944 (24 graphs)      \\ 

         Edges & 57 mil &  495,924 & 289,003  & 287,209 & 163778 & 143,662 & 818716 \\
   
         Classes  & 41 & 5 & 2 & 10 & 15  &  8  & 121 (multilabel)  \\
         \bottomrule
    \end{tabular}
    }
\end{table}

\noindent\textbf{Baselines}. \textbf{Transductive learning:}
We compar \mname with the following baseline methods. (1) The original graph attention networks (GAT)~\citep{velickovic2018graph}, (2) SparseGAT \citep{ye2019sparse} that learns edge coefficients to sparsify the graph, (3) random subsampling of edges, and (4) FastGCN~\citep{chen2018fastgcn} that is also designed for GNN speedup. Note that we compare SparseGAT in a transductive setting, whereas the original paper \citep{ye2019sparse} uses an inductive setting. We thus demonstrate that \mname can handle the full input graph, unlike any previous attention based baseline method. \textbf{Inductive learning:} For this task, we compare with both GAT \citep{velickovic2018graph} and GraphSAGE \citep{hamilton2017inductive}. More importantly, for both inductive and transductive tasks, we show that a uniform random subsampling of edges results in a drop in performance, where as \mname does not.

\noindent\textbf{Evaluation setup} and \textbf{Model Implementation Details} are provided in Section.~\ref{sec:appendex-detail} in the appendix.

\subsection*{Q1. \mname provides faster training with state-of-the-art accuracy. }

Our first experiment is to study \mname on the accuracy and time performance of attention based GNNs in  node classification.  We sample  $q= \mathrm{int}(0.16 N\log N/ \epsilon^2 )$  number of edges from the distribution $p_e$ with replacement, as described in Section \ref{subsec:Reff}. 

\textbf{Transductive learning:} In this setting, we assume that the features of all the nodes in the graph, including train, validation and test nodes are available, but only the labels of  training nodes are available during training, similar to \citep{velickovic2018graph}. First, we provide a direct comparison between \mname and the original GAT model and report the results in Table \ref{tab:direct_comparison}. As can be observed from the results, \mname achieves the same test accuracy as the full GAT  across all  datasets, while being dramatically faster:   we are able to achieve up to \textbf{5x} on GPU (\textbf{10x} on CPU) speedup.

We then compare  \mname with the following baselines: SparseGAT \citep{ye2019sparse}, random subsampling of edges and FastGCN \citep{chen2018fastgcn}  in Table \ref{tab:acc_attn_gnn}. SparseGAT uses the attention mechanism to learn embeddings and a sparsifying mask on the edges. We compare the training time per epoch for the baseline methods against \mname in Figure \ref{fig:temporal comparison}. The results shows that \textbf{\mname matches state-of-the-art accuracy (F1-score), while being much faster}.  While random subsampling of edges leads to a model that is as fast as ours but with a degradation in accuracy performance. \mname is also faster compared to FastGCN on some large datasets, even though FastGCN does not compute any attention coefficients. Overall the classification accuracy of \mname remains the same (or sometimes even improves) compared to standard GAT, while the training time reduces drastically. This is most evident in the case of the Reddit dataset, where the vanilla GAT model \textbf{runs out of memory} on a machine with 128GB RAM and a Tesla P100 GPU when computing attention coefficients over $57$ million edges, while \mname can train with 10 seconds per epoch.

\begin{table}[h!]
\centering
\caption{Comparison of \mname with GAT~\citep{velickovic2018graph}.}
\label{tab:direct_comparison}
\resizebox{\textwidth}{!}{
\begin{tabular}{ll|c c c c c c c}
\toprule
   Metric & Method & Reddit & Phy & Git & Comp & CS & Photo \\
     \midrule
   \multirow{3}{*}{F1-micro}&  GAT & OOM & 0.94\rpm 0.001 & 0.86\rpm0.000 & 0.89\rpm 0.004 & 0.89\rpm0.001 & 0.92\rpm0.001 \\
   ~ & \mname-const-0.5 & 0.93\rpm0.000 & 0.94\rpm0.001 & 0.86\rpm0.001 & 0.88\rpm0.004 & 0.88\rpm0.001 & 0.91\rpm0.002 \\
  ~ &  \mname-const-0.9 & 0.88\rpm0.001 & 0.94\rpm0.002 & 0.85\rpm0.002 & 0.86\rpm0.002 & 0.88\rpm0.004  & 0.89\rpm0.002 \\
     \midrule
   \multirow{3}{*}{GPU Time (s)}&  GAT & OOM & 3.67 & 3.71 & 2.93 & 1.61 & 1.88  \\
    ~&\mname-const-0.5 & 7.73 & 1.96 & 2.06 & 0.83 &  1.14 & 0.66 \\
    ~&\mname-const-0.9& 4.07  &1.80 & 1.63 & 0.50& 0.95 & 0.40  \\
    \midrule
     \multirow{3}{*}{CPU Time (s)}&GAT & OOM & 25.19 & 18.05 & 16.58 & 3.59 & 6.57\\
     ~& \mname-const-0.5 & 178.79 & 4.42 & 5.92 & 2.58 & 2.27 & 1.72\\
    ~ & \mname-const-0.9 & 41.42 & 2.70 & 3.44 & 1.22 & 1.48 & 0.77 \\
    \midrule
   \multirow{2}{*}{\% Edges redu.}&\mname-const-0.5 & 97.03\% & 74.04\% & 55.6\% & 78.6\% & 67\% & 79.3\%  &\\ 
   ~& \mname-const-0.9 & 99.03\% & 88.48\% & 79.77\%  & 91.99\% & 83.80\% &  91.86\% \\
    \bottomrule
\end{tabular}
}
\vskip -10pt
\end{table}

\begin{figure}[h!]
\vspace{4pt}
    \centering
    \includegraphics[scale=0.55]{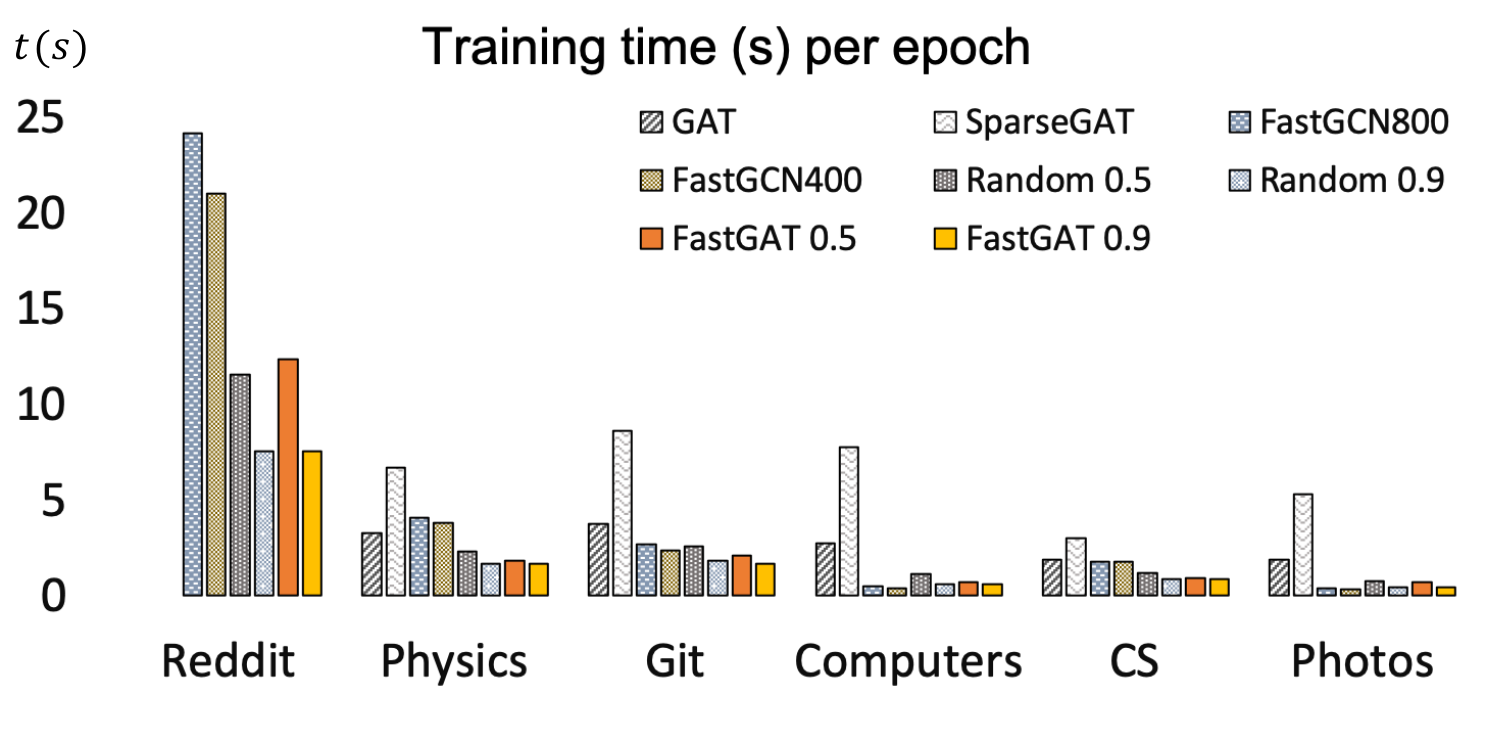}
    \vskip -8pt
    \caption{Training time per epoch for our \mname model and  baselines. \mname has a significantly lower training time, while also preserving classification accuracy (in Table.~\ref{tab:acc_attn_gnn}). On the Reddit dataset, \textcolor{purple}{both GAT and sparseGAT models run out of memory} in the transductive setting and hence we do not include bar graphs for them in the plot above.}
    \label{fig:temporal comparison}
    \vskip -10pt
\end{figure}

\begin{table}[h!]
    \centering
      \caption{\textbf{Transductive learning:} While significantly lower training time (in Fig.~\ref{fig:temporal comparison}), \mname still has comparable and sometimes better accuracy than other models. Computationally intensive model sparseGAT runs out of memory on the Reddit dataset. }
    \label{tab:acc_attn_gnn}
    \resizebox{\textwidth}{!}{
    \begin{tabular}{l|c c c c c c}
    \toprule
        Method& Reddit  & Coauthor-Phy & Github & Amazon-Comp & Coauthor-CS  & Amazon-Photos  \\
        \midrule
        FastGCN-400 \citep{chen2018fastgcn} & 0.77\rpm0.001 & \textbf{0.95\rpm0.001} & \textbf{0.86\rpm0.001} & 0.81\rpm0.002 & \textbf{0.92\rpm0.001} & 0.86\rpm0.005 \\
        FastGCN-800 \citep{chen2018fastgcn} & 0.81\rpm0.001 & \textbf{0.95\rpm0.001} & \textbf{0.86\rpm0.001} & 0.80\rpm0.009 & 0.92\rpm0.001 & 0.86\rpm0.007\\
        GAT rand 0.5 & 0.87\rpm0.002 & 0.94\rpm0.001 & 0.85\rpm0.001 & 0.86\rpm0.003 & 0.88\rpm0.002 & 0.90\rpm0.005\\
        GAT rand 0.9 & 0.82\rpm0.005 & 0.93\rpm0.000 & 0.84\rpm0.002 & 0.83\rpm0.004 & 0.88\rpm0.002 & 0.89\rpm0.004\\
        sparseGAT \citep{ye2019sparse} &OOM & 0.94\rpm 0.004 & 0.84\rpm 0.002 &0.85\rpm0.001 & 0.88\rpm 0.002 & \textbf{0.91\rpm 0.003}\\
         \midrule
        \mname-0.5 & \textbf{0.93\rpm 0.001} & \textbf{0.95\rpm 0.001} & \textbf{0.86\rpm 0.002}  &  \textbf{0.88 \rpm 0.003}& 0.88\rpm 0.002  & 0.90\rpm 0.002 \\
          \mname-0.9&  0.91\rpm 0.002 & 0.94\rpm 0.002  & \textbf{0.86\rpm 0.000} & 0.85\rpm0.002 & 0.88\rpm 0.002 & 0.90\rpm 0.002\\
        \mname-const-0.5 & \textbf{0.93\rpm0.000} & 0.94\rpm0.001 & \textbf{0.86\rpm0.001} & \textbf{0.88\rpm0.004} & 0.88\rpm0.001 & \textbf{0.91\rpm0.002}  \\
         \mname-const-0.9& 0.88\rpm0.001 & 0.94\rpm0.002 & 0.85\rpm0.002 & 0.86\rpm0.002 & 0.88\rpm0.004  & 0.89\rpm0.002\\

         \bottomrule
    \end{tabular}}
    \vskip -10pt
\end{table}

\blfootnote{In Tables \ref{tab:direct_comparison} and \ref{tab:acc_attn_gnn}, FastGCN-400 denotes that we sample 400 nodes in every forward pass, as described in \citep{chen2018fastgcn} (similarly, in FastGCN-800, we sample 800 nodes). \mname-0.5 denotes we use  $\epsilon=0.5$. GAT rand 0.5 uses random subsampling of edges, but keeps the same number of edges as \mname-0.5}

\textbf{Inductive learning}.
\mname can also be applied to the inductive learning framework, where features of only the training nodes are available and training is performed using the subgraph consisting of the training nodes. To show the utility of \mname in such a setting, we use the Protein-Protein interaction (PPI) dataset (\citep{zitnik2017predicting}). Our model parameters are the same as in \citep{velickovic2018graph}, but we sparsify each of the 20 training graphs before training on them. The other 4 graphs are used for validation and testing (2 each). We use $\epsilon = 0.25, 0.5$ in our experiments, since the PPI dataset is smaller than the other datasets. We report the experimental results in Table \ref{tab:PPI}. \mname clearly outperforms the baselines like GraphSAGE and uniform subsampling of edges. While it has the same accuracy performance as the GAT model (which is expected), it has a much smaler training time, as reported in Table \ref{tab:PPI}. 
 \begin{table}[h!]
    \centering
    \caption{\textbf{Inductive learning on PPI Dataset} The training time per epoch for the full GAT method is around 45.8s, whereas \mname requires only about 29.8s when $\epsilon = 0.25$ and $20$s when $\epsilon = 0.5$. \mname achieved the same accuracy as GAT, and outperforms the other baselines, while also begin computatioanally efficient. The F1 score for GraphSAGE was obtained from what is reported in \cite{velickovic2018graph}. }\label{tab:PPI}
    \begin{tabular}{l|cc}
    \toprule
         Method & F1 score \\
         \midrule
         GAT & \textbf{0.974 \rpm 0.002} \\
         GraphSAGE* & 0.768 \rpm 0.000   \\
         GAT rand-0.25 & 0.95 \rpm 0.005\\
         GAT rand-0.9 &  0.7\rpm 0.004\\
         \midrule
        \mname-const-0.25 &  \textbf{0.974 $\pm$ 0.003}\\
         \mname-const-0.5 & 0.800 $\pm$ 0.005\\
         \bottomrule
    \end{tabular}
\end{table}




\subsection*{Q2. \mname can be applied to other attention based graph inference methods.}

Finally, we study if \mname is sensitive to the particular formulation of the attention function. There have been alternative formulations proposed to capture pairwise similarity. For example, \citep{thekumparampil2018attention}
proposes a \textbf{cosine similarity} based approach, where the attention coefficient of an edge is defined in Eq.~\eqref{eq:cosattn},
\begin{equation}
    \alpha_{ij}^{(\ell)}=\underset{j \in \mathcal{N}_i}{\mathrm{softmax}} \left ( [\beta^{(\ell)} \cos (\vh_i^{(\ell)}, \vh_j^{(\ell)} ) ] \right )
    \label{eq:cosattn}
\end{equation} where $\beta^{(\ell)}$ is a layer-wise learnable parameter and $\cos(\bm{x},\bm{y}) = \bm{x}^\top \bm{y}/ \norm{\bm{x}}\norm{\bm{y}}$. Another definition is proposed in \citep{zhang2018gaan} (\textbf{GaAN}:Gated Attention Networks), which defines attention as in Eq.~\eqref{eq:gaanattn}, 
\begin{equation}
    \alpha_{ij}^{(l)} = \underset{j \in \mathcal{N}_i}{\mathrm{softmax}} \  \langle \mathrm{FC}_{src} (\vh_i^{\ell}), \mathrm{FC}_{dst} (\vh_j^{\ell}) \rangle
    \label{eq:gaanattn}
\end{equation} where $\mathrm{FC}_{src}$ and $\mathrm{FC}_{dst}$ are 2-layered fully connected neural networks. 

We performed similar experiments on these attention definitions. Tables. \ref{tab:cosine} 
confirmed that \mname generalizes to different attention functions.
Note that the variability in accuracy performance across Tables \ref{tab:direct_comparison} and \ref{tab:cosine}  comes from the different definitions of the attention functions and not from \mname. Our goal is to show that given a specific GAT model, \mname can achieve similar accuracy performance as that model, but in much faster time. 

\begin{table}[h!]
\centering
\caption{Comparison of full and sparsified graphs. }

\label{tab:cosine}
\resizebox{\textwidth}{!}{
\begin{tabular}{ll|c c c c c c }
\toprule
  \multicolumn{2}{c}{\textbf{ \mname for  Cosine Similarity } } & Reddit & Phy & Github & Comp & CS &  Photo    \\
     \midrule
   \multirow{3}{*}{F1-micro}&  Full graph & OOM & 0.96\rpm0.001 & 0.86\rpm0.001 & 0.88\rpm0.003 &  0.91\rpm0.001 & 0.93\rpm0.001  \\
   ~ & $\epsilon = 0.5 $ & 0.93\rpm0.000 & 0.95\rpm0.001 & 0.86\rpm0.001  & 0.88\rpm0.003 & 0.91\rpm0.001 & 0.92\rpm0.002 \\
   ~ & $\epsilon=0.9$ & 0.88\rpm0.001 & 0.95\rpm0.003 & 0.86\rpm0.002 & 0.87\rpm0.004 & 0.89\rpm 0.002 & 0.90\rpm0.002  \\
\midrule
    \multirow{3}{*}{Time/epoch (s)}&  Full graph  & OOM & 3.39 & 3.89 & 2.80 & 1.54 & 1.89 
   \\
   ~ & $\epsilon = 0.5 $ & 8.02 & 1.99 &2.3 & 0.77 &  1.11 & 0.724 \\
   ~ & $\epsilon = 0.9$ & 4.46 & 1.79 & 1.69 & 0.48 & 0.91 & 0.435  \\
   \midrule\midrule


  \multicolumn{2}{c}{\textbf{ \mname for  GaAN } } & Reddit & Phy & Github & Comp & CS &  Photo    \\
     \midrule
   \multirow{3}{*}{F1-micro}&  Full graph & OOM &  0.94\rpm 0.001 &  0.86\rpm 0001 & 0.86\rpm0.002 & 0.873\rpm 0.001 & 0.90\rpm0.001   \\
    ~ & $\epsilon=0.5$ & 0.92\rpm 0.001 & 0.94\rpm 0.002 &  0.86\rpm 0.002 &0.84\rpm0.001 & 0.87\rpm 0.002  & 0.89\rpm 0.003\\
   ~ &$\epsilon=0.9$ & 0.87\rpm 0.001 & 0.93\rpm 0.001 & 0.86\rpm 0.002 &0.82\rpm 0.002 &  0.84 \rpm 0.001 & 0.88\rpm 0.002 \\
    \midrule
    \multirow{3}{*}{Time/epoch (s)}& Full graph & OOM & 3.55 &  3.87 &  2.77 & 1.60 & 1.90  
   \\
   ~ & $\epsilon=0.5$ & 8.10 & 1.93 &  2.10 & 0.84 & 1.10 & 0.67  \\
   ~ & $\epsilon=0.9$ & 4.47  & 1.76& 1.54 & 0.50 &  0.90& 0.41  \\
   \bottomrule
\end{tabular}
}
\vskip -10pt
\end{table}

\section{Conclusion}

In this paper, we introduced \mname, a method to make attention based GNNs lightweight by using spectral sparsification. 
We theoretically justified our \mname algorithm. 
\mname can significantly reduce the computational time across multiple large real world graph datasets while attaining state-of-the-art performance.


\bibliographystyle{iclr2021_conference}
\bibliography{FastGAT}

\appendix

\newpage


\appendix

{SUPPLEMENTARY MATERIAL - FAST GRAPH ATTENTION NETWORKS USING EFFECTIVE RESISTANCE BASED GRAPH SPARSIFICATION}

\section{Proofs}


\subsection{Reinterpretation of attention based GNNs as graph convolution models}

For our analysis, we assume the attention vector $\bm{a}$ defined in \citep{velickovic2018graph} is symmetric: $
    \bm{a}[1:N] = \bm{a}[N+1:2N].$ This results in a  symmetric attention coefficient matrix and  the analysis is simpler. Also note that the symmetric attention functions are used in practice as well \citep{thekumparampil2018attention}. Our first result is on the equivalence of  GAT  and graph convolutions. As stated in \eqref{eq:gconv}, graph convolution with kernel $\bm{g}$ can be defined using the eigenvectors and eigenvalues of the symmetric normalized graph Laplacian matrix $\Lsym$. Similarly we can define a convolution operation using an alternative of the normalized Laplacian, known as the random walk normalized Laplacian matrix, which is defined as:
\begin{equation*}
    \Lrw = \bmD^{-1}\bm{L} = \bm{I} - \bmD^{-1}\bmA = \bmD^{-1/2}\Lsym \bmD^{1/2}.
\end{equation*} The convolution operation with $\Lrw$ can then be defined as in Eq.~\eqref{eq:newconv}
\begin{equation}
    \bm{g}\newconv \bm{x} = \bmD^{-1/2} \bm{U} \bm{g}(\bm{\Lambda}) \bm{U}^T \bmD^{-1/2} \bm{x}.
    \label{eq:newconv}
\end{equation}

\begin{proposition}
Each layer in the GAT model defines a new, \textbf{layer-dependent graph adjacency matrix} $\mathbf{\Gamma}^{(l)}$ and a corresponding degree matrix $\mathbf{\Gamma_{\bmD}}^{(l)}$. Each layer then computes a first order approximation of the convolution operator $\widehat{\star}$ defined in \eqref{eq:newconv} using $\mathbf{\Gamma}^{(l)}$ and $\mathbf{\Gamma_{\bmD}}^{(l)}$:
\begin{equation*}
    \bm{H}^{(l+1)} = \sigma(\mathbf{ \Gamma_{\bmD}}^{(l)^{-1}} \mathbf{\Gamma}^{(l)} \bm{H}^{(l)} \bm{W}^{(l)}) \approx \sigma ( \bm{g}_{\mathbf{\Gamma}^{l}, \mathbf{\Gamma_D}^{(l)}}^{(l)}\newconv (\bm{H}^{(l)} ) )
\end{equation*} where $\sigma$ is the non-linearity and $\bm{g}_{\mathbf{\Gamma}^{(l)}, \mathbf{\Gamma_D}^{(l)}}^{(l)}$ is the convolutional kernel characterized by $\mathbf{\Gamma}^{(l)}, \mathbf{\Gamma_D}^{(l)}$.
\label{prop:GAT_newconv}
\end{proposition}

Proposition \ref{prop:GAT_newconv} establishes a direct equivalence between the graph convolution and the GAT model. This interpretation also shows that the GAT model applies Laplacian smoothing to node based features \citep{taubin1995signal, li2018deeper}. Such a connection between spectral operations and attention based graph neural networks provides directions for theoretical analysis of attention GNNs. 

\subsection{GAT model is equivalent to layer-wise convolution}

Consider a single layer of the GAT model. Let $\bm{H} \in \R^{N\times D}$ be the input feature matrix to a single GAT layer, let $\bm{W}\in \R^{D\times F}$ denote the weight matrix, let $\bm{C}\in \R^{D\times 2}$ be such that $\bm{C} = [ \bm{a}(1:N) \ \bm{a}(N+1:2N) ]$, where $\bm{a} \in \R^{2N}$ denotes the attention coefficient vector as defines in \citep{velickovic2018graph}. Let $\bm{A} \in \R^{N \times N}$ be the graph adjacency matrix and let $\widetilde{\bm{A}}=\bm{A}+\bm{I_N}$. For a given graph, if $\bm{D}$ represents the degree matrix, then  $\bm{D}^{-1}\bm{A}$ is simply the state transition matrix of a random walker on the graph. 

We can further define a matrix $\bm{Q} \in \R^{N \times 2}$ as 
\begin{equation}
    \bm{Q} = \bm{HWC}.
\end{equation}
Further, let us define $e_{ij}$ (as in \citep{velickovic2018graph}) as 
\begin{equation}
    e_{ij} = a^\top [\bm{Wh_i} ||\bm{ Wh_j}]
\end{equation} where $h_i$ is the $i$th row of $\bm{H}$. Then, we can see that the vector $\bm{e}_i$ where $e_i(j) = e_{ij}$ is given as 

\begin{equation}
    \bm{e}_i = \begin{bmatrix}
    \bm{\wa}(:,i) & \text{diag}(\bm{\wa} (:,i))
    \end{bmatrix} \begin{bmatrix} \bm{Q}(i,:) & \mathbf{0}_{1\times 2} \\
    \mathbf{0}_{N\times2} & \bm{Q} \end{bmatrix} \begin{bmatrix} 1 \\ 0 \\ 0 \\ 1 \end{bmatrix}.
\end{equation}
Hence, we can express the matrix of attention coefficients, before the softmax operation, $\Gamma$ as 

\begin{equation}
    \bm{\Gamma} = f( \widehat{\bm{Q}}^\top \widehat{\bm{A}} )
\end{equation} where 

\begin{equation}
\resizebox{.9 \textwidth}{!} 
{
$   \bm{ \widehat{Q}}  = 
    \begin{bmatrix}
    \bm{Q}(1,:) &  &  &  &  & & \\
     & \bm{Q}(2,:) &  &  &   & &\\
     & & \ddots &  &   &   \\
     & &  & \bm{Q}(N,:) &   &\\
     &  &  & & \bm{ Q }&   &  \\
     &  &  &  & & \ddots &  \\
     &  & & &  & & \bm{Q}\\
    \end{bmatrix} 
    \begin{bmatrix}
    1 & 0 & \cdots \\
    0 & 0 & \cdots \\
     0 & 1 & \cdots \\
    \mathbf{0}_{N-1} & \mathbf{0}_{N-1} & \cdots \\
    0 & 0 & \cdots \\
    1 & 0 & \cdots \\
    0 & 0 & \cdots \\
    0 & 1 & \cdots \\
    \vdots & \vdots & \cdots
    \end{bmatrix} 
    , \ \widehat{\bm{A}}  = 
    \begin{bmatrix} 
    \bm{\wa} \\
    \text{diag}( \bm{\wa} (:,1)) \\
    \vdots \\
    \text{diag}( \bm{\wa} (:,N)) \\
    \end{bmatrix}$
}
\end{equation} 

and $f = \exp (\text{LeakyReLU}(\cdot))$. Further, let $\bm{\Gamma_D}=\mathrm{diag}(\bm{\Gamma} \mathbf{1}_N)$.

The GAT layer update can be expressed as, 
\begin{align}
    \widehat{\bm{H}}&= \sigma \left ( \text{diag} \left ( f( \widehat{\bm{Q}}^\top \widehat{\bm{A}} ) \mathbf{1}_N \right )^{-1}  f( \widehat{Q}^\top \widehat{\bm{A}} ) \bm{H} \bm{W}\right ) \label{eq:new_int_compare}\\
        \widehat{\bm{H}} &= \sigma \left ( \bm{\Gamma_D}^{-1} \bm{\Gamma} \bm{H} \bm{W}\right ) \label{eq:new_int}
\end{align}

Given a new graph $\bm{G}_{\Gamma}(\mathcal{E}, \mathcal{V})$ with $\bm{ \Gamma}$ being the adjacency matrix. Let $\bm{ \Gamma}_D$ be the corresponding degree matrix. Then, the random walk normalized Laplacian is defined as 
\begin{equation}
   \bm{ L_{\mathrm{rw}}} = \bm{I} - \bm{\Gamma_D}^{-1} \Gamma
\end{equation}
Note that although $\bm{\Lrw}$ is asymmetric, it is similar to the symmetric normalized Laplacian matrix:
\begin{equation}
    \bm{\Lsym} = \bm{I}-\bm{\Gamma_D}^{-1/2}\bm{\Gamma} \bm{\Gamma_D}^{-1/2} = \bm{\Gamma}^{1/2}\bm{\Lrw} \bm{\Gamma}^{-1/2}.
    \label{eq:Lsym_def}
\end{equation}Hence, $\bm{ \Lrw}$ has real eigenvalues and match with those of $\bm{ \Lsym}$. The corresponding eigenvectors of the two matrices are also related: If $\bm{v}$ is an eigenvector of $\bm{\Lsym}$, then $\bm{\Gamma_D}^{-1/2}v$ is an eigenvector of $\bm{\Lrw}$. Using this, we can define a new convolution operator as 
\begin{equation}
    \bm{g_{\theta}}\newconv \bm{h} = \bm{\Gamma_D}^{-1/2} \bm{U_{\mathrm{sym}}} \bm{g_\theta}(\bm{\Lambda_{\mathrm{sym}}}) \bm{U_{\mathrm{sym}}}^T \bm{\Gamma_D}^{1/2} \bm{h}
\end{equation}

Then, using the Chebychev polynomial approximation similar to \citep{kipf2016semi}, we can show that for a given feature vector $\bm{h}$, we can get a first order approximation to the operation $\bm{g_{\theta}} \newconv \bm{h}$ as 
\begin{equation}
    \bm{g_{\theta}} \newconv \bm{h} \approx \bm{\Gamma_D}^{-1} \bm{\Gamma} \bm{h} \ w.
\end{equation} For multiple output features, the new graph convolution operation has a first order approximation as 
\begin{equation}
    \bm{g_{\theta}}\newconv \bm{H} \approx \bm{\Gamma_D}^{-1} \bm{\Gamma} \bm{H} \ \bm{W}
\end{equation} which matches exactly with \eqref{eq:new_int}. This shows that the model defined in \citep{velickovic2018graph} is similar to a GCN model, but defined layer-wise.

\subsection{Spectral sparsification preserves graphs convolutional features}
\label{subsec: proofs}
In this section, we show that the features learnt by graph convolution based neural networks are preserved when spectral sparsification techniques are applied to the original data graph. We use the following notation: $\bm{\Lsym}$ is as defined in \eqref{eq:Lsym_def} and denotes the symmetric normalized Laplacian matrix of a graph, $\bm{\Lsyms}$ denotes the symmetric normalized Laplacian matrix of the corresponding spectrally sparsified graph with a parameter of $\epsilon$. Similalry, we use $\Lrw$ to denote the random walk normalized Laplacian matrix of a graph and $\Lrws$ to denote the corresponding  random walk normalized Laplacian matrix of the spectrally sparsified graph. 

\noindent \textbf{ Spectral sparsification and the GCN model}
Consider the graph convolution network architecture proposed in \cite{kipf2016semi}. We assume that the non-linearity $\sigma(\cdot)$ used in Lipshitz continuous with a Lipschitz constant $\ell_\sigma$. For a single neural network layer, let the input features be $\bm{H} \in \R^{N \times D}$, let the weight matrix be $\bm{W} \in \R^{D \times F}$, where $F$ is the number of output features. Then, the new set of features computed by the GCN model is 
\begin{align}
    \bm{\widehat{H}} = \sigma \left( (\Lsym - \bm{I}) \bm{H W}\right )
\end{align} and the corresponding set of features computed by the GCN model using a spectrally sparsified graph are given as 
\begin{align}
    \bm{\widehat{H}'} = \sigma \left( (\Lsyms -\bm{I}) \bm{H W}\right )
\end{align}

\textbf{Proof of Theorem \ref{thm:GCN}}

\begin{proof}
We first characterize the spectral norm error between the corresponding graph Laplacians $\Lsym$ and $\Lsyms$ and then use the bound to prove Theorem \ref{thm:GCN}. We use $\bmD$ and $\bmDs$ to denote the degree matrices of the full and the spectrally sparsified graphs. 

Since both $\Lsym$ and $\Lsyms$ are symmetric and positive semidefinite, we have, 
\begin{align}
    \norm{\Lsym - \Lsyms} & = \max  ( |\lambda_{\max} (\Lsym - \Lsyms)|, |\lambda_{\min}(\Lsym - \Lsyms)|  ) \nonumber \\
     & = \underset{\bm{x \in \R^{N} \colon \norm{\bm{x}}=1}}{\sup}{|\bmx^\top (\Lsym - \Lsyms)\bmx|} \nonumber
\end{align}
We then have 
\begin{align}
    |\bmx^\top (\Lsym - \Lsyms)\bmx| & = | \bmx^\top \Lsym \bmx - \bmx^\top \bmD^{-1/2}\bmDs^{1/2}\Lsyms \bmDs^{1/2} \bmD^{-1/2}\bmx   \\
      & \ \ + \bmx^\top \bmD^{-1/2}\bmDs^{1/2}\Lsyms \bmDs^{1/2} \bmD^{-1/2}\bmx - \bmx^\top\Lsyms\bmx \nonumber| \nonumber \\
    & \leq | \bmx^\top \Lsym \bmx - \bmx^\top \bmD^{-1/2}\bmDs^{1/2}\Lsyms \bmDs^{1/2} \bmD^{-1/2}\bmx | \nonumber \\
    & \ \ + |\bmx^\top \bmD^{-1/2}\bmDs^{-1/2}\Lsyms \bmD^{-1/2}\bmDs^{-1/2} \bmx  - \bmx^\top\Lsyms\bmx | \nonumber \\
\end{align}
Taking supremum on both sides, we get 
\begin{align}
    \underset{\bm{x \in \R^{N} \colon \norm{\bm{x}}=1}}{\sup}{|\bmx^\top (\Lsym - \Lsyms)\bmx|} & \leq \epsilon  \underset{\bm{x \in \R^{N} \colon \norm{\bm{x}}=1}}{\sup} |\bmx^\top \Lsym \bmx | + \nonumber \\ & \ \ \norm{\bmD^{-1/2}\bmDs^{-1/2}\Lsyms \bmD^{-1/2}\bmDs^{-1/2} - \Lsyms} \nonumber \\
    & \leq \epsilon \norm{\Lsym} + \\ & \ \ \norm{\bmD^{-1/2}\bmDs^{-1/2}\Lsyms \bmD^{-1/2}\bmDs^{-1/2} - \Lsyms} \nonumber \\
    & \leq \epsilon \norm{\Lsym} + 3\epsilon \norm{\Lsym}\nonumber \\
    & \leq 4\epsilon \norm{\Lsym}
\end{align}
where we assume that $3\epsilon^2 + \epsilon^3 < \epsilon$, which holds true for small $\epsilon$.

We then have the final result as below. Let $\ell_\sigma$ be the Lipschitz constant of the non-linearity $\sigma$.

\begin{align}
    \norm{\bm{\widehat{H}} - \bm{\widehat{H}'}}_F & \leq \ell_\sigma 4 \epsilon \norm{(\Lsym - \Lsyms) \bm{H}\bm{W}}_F  \\
    & 4\epsilon \norm{\Lsym - \Lsyms}\norm{\bm{H}\bm{W}}_F
\end{align} where, we use $\ell_\sigma=1$ for ReLU or ELU non-linearity, and use inequality 
\[ \norm{\bm{AB}}_F \leq \norm{\bm{A}} \norm{\bm{B}}_F \] for any two matrices $\bm{A}$ and $\bm{B}$.

\end{proof}

Theorem \ref{thm:GCN} shows that if two GCN models that use the full and spectrally sparsified graphs have the same initialization $\bm{W}$, then the correponding feature updates are close in a Frobenius norm sense. Although we have not explored the dynamics or training, we strongly believe that similar bounds can be obtained on the gradients of the network parameters and in turn on the gradient descent updates.

\noindent \textbf{Spectral sparsification and the GAT model}
We can now consider the graph attention network model proposed in \citep{velickovic2018graph}. With $\bm{H}$ and $\bm{W}$ as defined in the previous section, the feature update equations for the GAT model using the full and the spectrally sparsified graphs are given as 
\begin{align}
    \bm{\widehat{H}} = \sigma ( (\bm{\Lrw - I})HW )  \nonumber \\
    \bm{\widehat{H}'} = \sigma( (\bm{\Lrws - I}) HW )
    \label{eq:GAT_layer_update}
\end{align}
Due to Equations \eqref{eq:Lsym_def}, we can rewrite the above equations as 
\begin{align}
    \bm{\widehat{H}} = \sigma ( \bmD^{-1/2}(\bm{\Lsym - I})\bmD^{1/2} HW ) \label{eq:full_gat} \\
    \bm{\widehat{H}'} = \sigma( \bmDs^{-1/2}(\bm{\Lsyms - I})\bmDs^{-1/2} HW ) \label{eq:GAT_Lsym}
\end{align}

As before, we first bound the error $\norm{\Lrw - \Lrws}$ and then use it to bound the error $\bm{\widehat{H}} - \bm{\widehat{H}'}$. 

\textbf{Proof of Theorem \ref{thm:GAT}}

Theorem \ref{thm:GAT} shows that if a layer-wise spectral sparsification of the graph is used to reduce the number of edges, then the feature updates computed by the sparse model are also preserved. Note that this requires sparsifying the graph in each layer separately with the weights in the adjacency matrix given by $\Gamma$. In the next section, we show that this expensive procedure of layer-wise sparsification can be repalced by a one-time spectal sparsification procedure for the binary node classification problem. 

We use the following Lemma to establish Theorem \ref{thm:GAT}: 
\begin{lemma}
Let $\bm{A} \in R^{N \times N}$ be any matrix and $\bm{D} \in \R^{N \times N}$ be a diagonal matrix with positive diagonal entries. Then, we have 
\begin{equation}
    \norm{\bm{A} - \bm{D}^{-1}\bm{A}\bm{D}} \leq \norm{\bm{A}} ( \norm{\bm{I} - \bm{D}^{-1}} + \norm{\bm{D}^{-1}}\norm{\bm{I} - \bm{D}} )
\end{equation}
\label{lm:sim_diff}
\end{lemma}
\begin{proof}
We have 
\begin{align}
    \norm{\bmA - \bmD^{-1}\bmA\bmD} & = \norm{\bmA - \bmD^{-1}\bmA + \bmD^{-1}\bmA - \bmD^{-1}\bmA\bmD} \nonumber \\
    & = \norm{\bmA(\bm{I}-\bmD^{-1}) + \bmD^{-1}\bmA (\bm{I} - \bmD)} \nonumber \\
    & \leq \norm{\bmA} \norm{\bm{I}-\bmD^{-1}} + \norm{\bm{I} - \bmD)}\norm{\bmD^{-1}} \norm{\bmA} \nonumber \\
    & \leq  \norm{\bm{A}} \left ( \norm{\bm{I} - \bm{D}^{-1}} + \norm{\bm{D}^{-1}}\norm{\bm{I} - \bm{D}} \right )
\end{align}
\end{proof}

\noindent \textbf{Proof of Theorem \ref{thm:GAT}:}
\begin{align*}
    \norm{\Lrw - \Lrws} & = \norm{\bmD^{-1/2}\Lsym \bmD^{1/2} - \bmDs^{-1/2}\Lsyms \bmDs^{1/2}} \\
    & =  \| \bmD^{-1/2}\Lsym \bmD^{1/2} -\bmD^{-1/2}\Lsyms \bmD^{1/2} \ + \\ & \hphantom{{}=\bmD^{-1/2}}  \bmD^{-1/2}\Lsyms \bmD^{1/2} - \bmDs^{-1/2}\Lsyms \bmDs^{1/2}\| \\
    & \leq \norm{\Lsym - \Lsyms} + \norm{\Lsyms - \bmD^{1/2}\bmDs^{-1/2}\Lsyms \bmDs^{1/2}\bmD^{-1/2}}
\end{align*}
Then, using Lemma \ref{lm:sim_diff}, we get
\begin{align*}
    \norm{\Lrw - \Lrws} & \leq  \norm{\Lsym - \Lsyms} + \epsilon(1+\epsilon)\norm{\Lsyms} \\
    & \leq (5\epsilon + \epsilon^2) \norm{\Lsym}\\
    &\leq 6\epsilon \norm{\Lsym}
\end{align*}

Further, from the feature update equations \eqref{eq:GAT_layer_update} and since $\sigma$ is Lipschitz continuous with Lipschitz constant $\ell_\sigma$, we have the final result as in the proof of Theorem \ref{thm:GCN}.

\subsection{Approximation of weight matrices}
\label{subsec:approx_weight}
Theorems \ref{thm:GCN} and \ref{thm:GAT} provide an upper bound on the feature updates obtained using the full and sparsified graphs under both GCN and GAT. A stronger notion of information preservation after sparsification is obtained by studying the weight matrices to see if the graph structure is retained after the sparsification. To this end, we would like to study the error
$\norm{\bm{W} - \bm{W_s}}$, 
where $\bm{W}$ and $\bm{W_s}$ are weight matrices of the neural network, learned using the full and the sparsified graph respectively in any given layer. Such a result shows whether the graph structure retained is sufficient for the GAT model to learn strong features. 

\begin{figure}[h]
    \centering
    \includegraphics[scale=0.25]{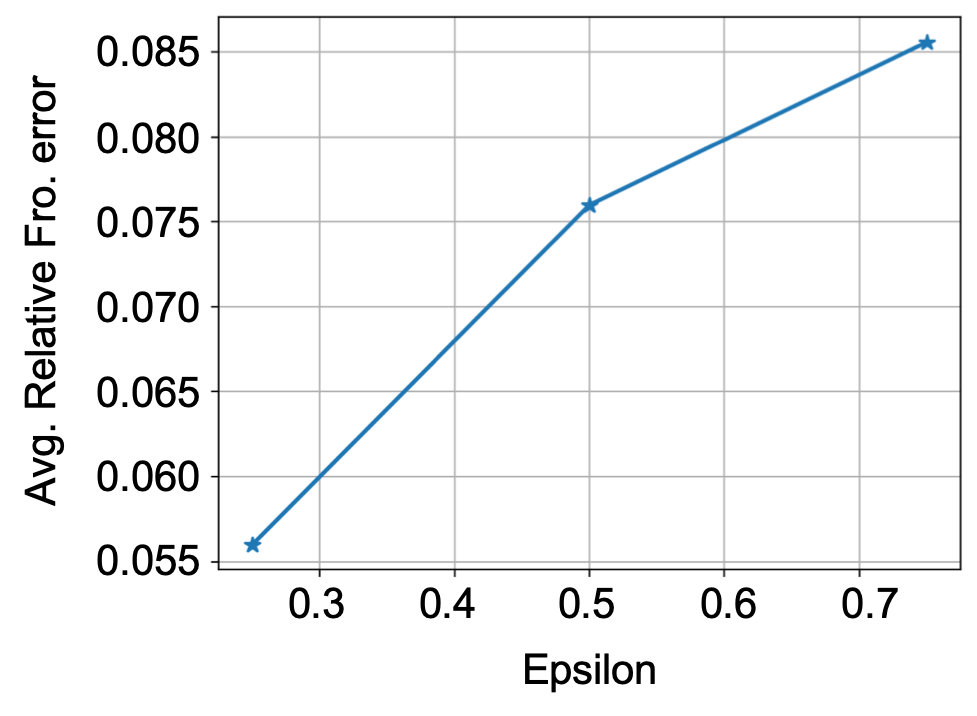}
    \includegraphics[scale=0.25]{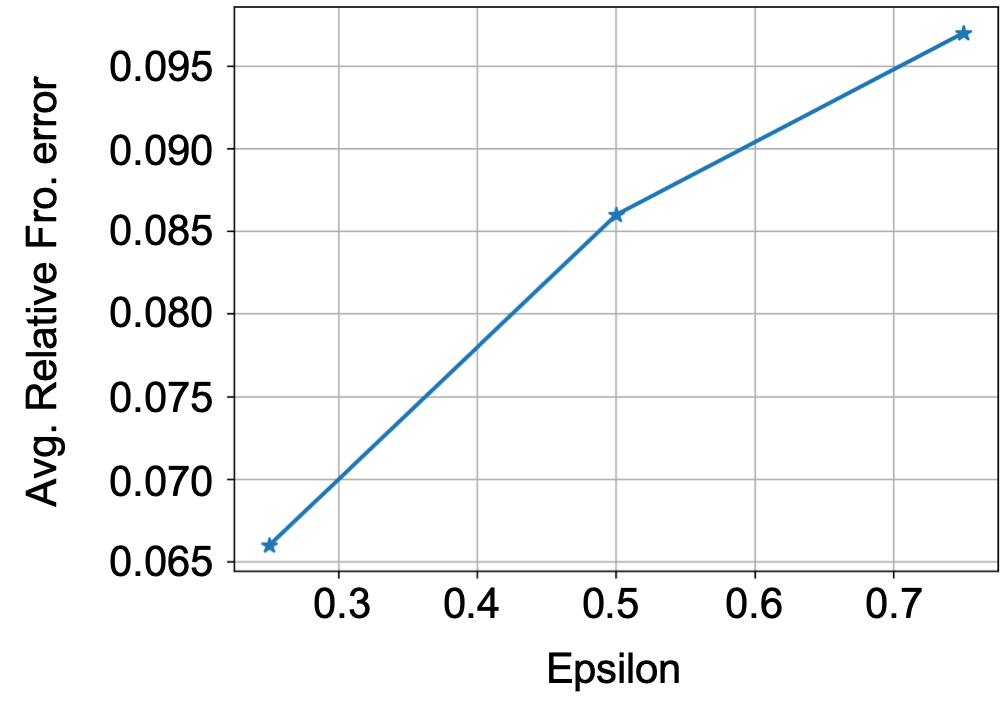}
    \includegraphics[scale=0.25]{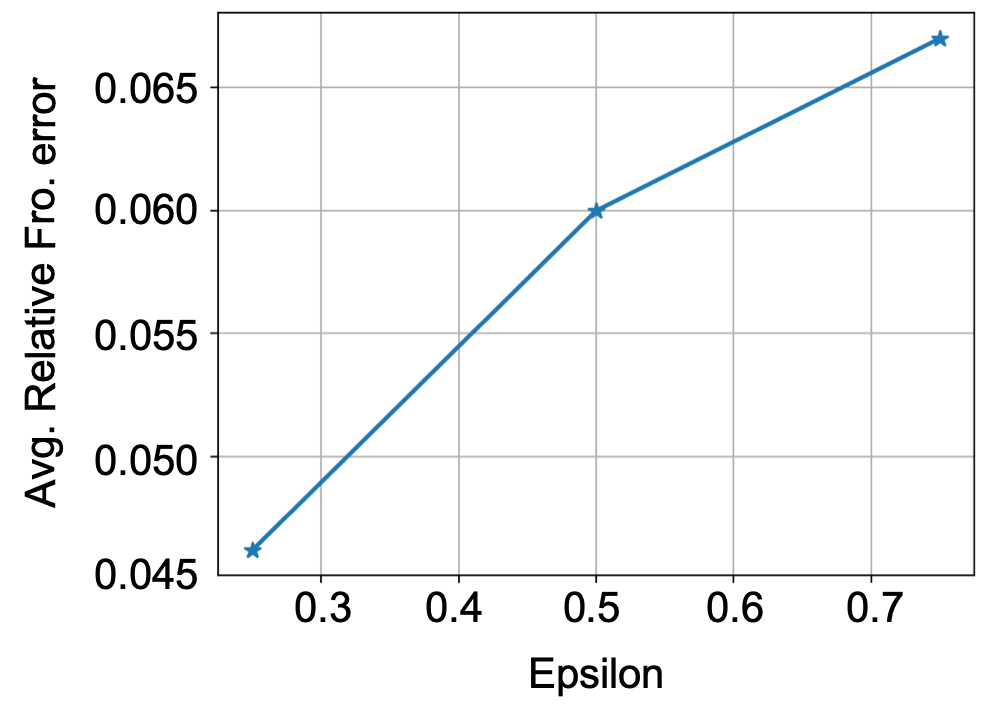}
    \includegraphics[scale=0.25]{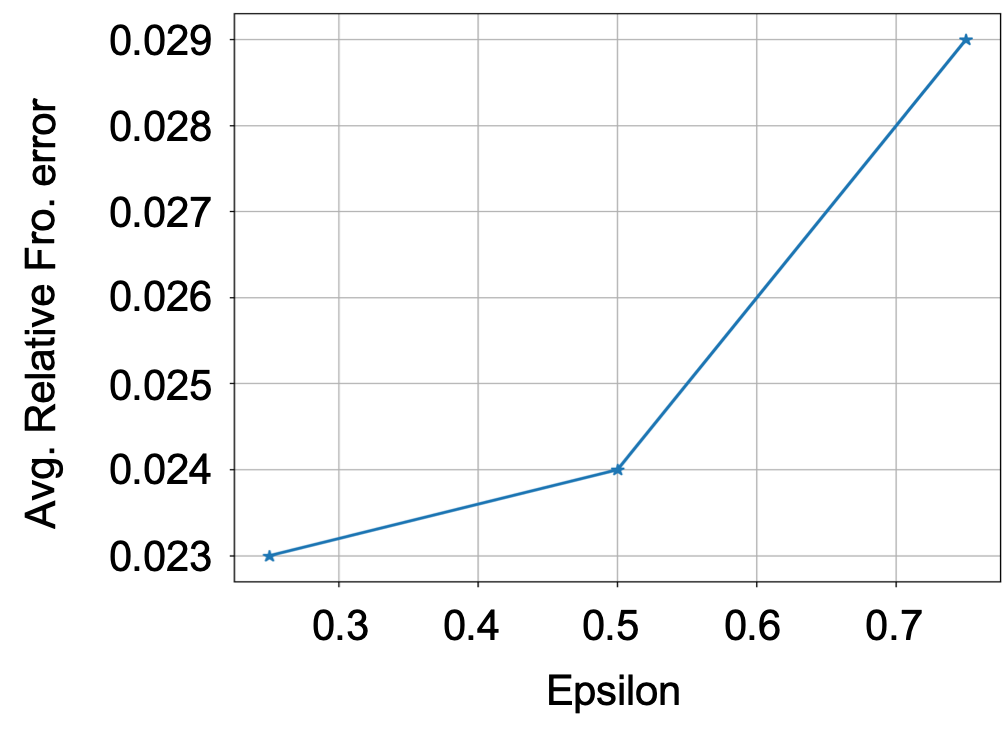}
    \includegraphics[scale=0.25]{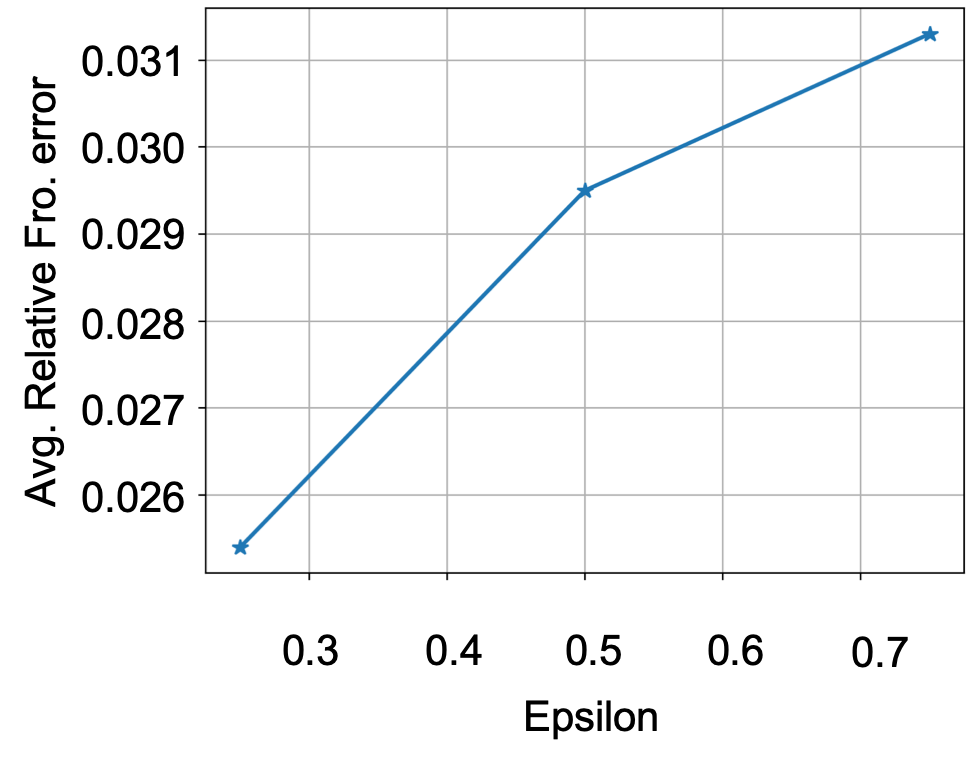}
    \caption{Relative Error between the learned weights without and with sparsification on Coautho-Phy, Github Social and Couathor-CS (top) and Amazon-Computer and Amazon-Photos (bottom) datasets.  We can see that the error is proportional to the $\epsilon$ parameter. Such a comparison was not possible on the Reddit dataset, since the model cannot be run on the full graph.}
    \label{fig:W_comp}
\end{figure}

To lend support to this claim, we  studied the difference between the weight matrices learned with and without spectral sparsification. We used three different datasets (Coautho-Phy, Github Social and Couathor-CS). In each case, we used three different values of $\epsilon$ ($0.25, \ 0.5, \ 0.75$). At each parameter setting, we performed 5 independent trials and averaged the relative Frobenius errors between the weight matrices and the attention function $a$ of all attention heads. We report the results in Fig.~\ref{fig:W_comp}. It is clear that the error between the learned matrices is proportional to the value of $\epsilon$ itself. This shows that the training process is highly stable with respect to spectral sparsification of the input graph.

\section{Details of Experiments}\label{sec:appendex-detail}

\subsection{Description of  datasets}
\label{subsec:data}

\textbf{Transductive learning tasks:} Amazon Computers and Amazon Photo are segments of the Amazon co-purchase graph~\citep{10.1145/2783258.2783381}, where nodes represent goods, edges indicate that two goods are frequently bought together, node features are bag-of-words encoded product reviews, and class labels are given by the product category. Coauthor CS and Coauthor Physics are co-authorship graphs based on the Microsoft Academic Graph from the KDD Cup 2016 challenge 3. Here, nodes are authors, that are connected by an edge if they co-authored a paper; node features represent paper keywords for each author’s papers, and class labels indicate most active fields of study for each author. For the Reddit dataset, we predict which community different Reddit posts belong to based on the interactions between the posts. The Github social dataset consists of Github users as nodes and the task is that of classifying the users as web or machine learning developers (binary classification). For all the above datasets, the task is that of node classification. Additionally, we have experiments on citation graphs: Cora, Citeseer and Pubmed. In these datasets, the nodes represent authors, edges represent mutual citations and the task is to categorize the authors into their fields of study. 

\textbf{Inductive learning tasks:} We use the Protein=Protein interaction dataset \cite{zitnik2017predicting} where the graphs correspond to different human tissues. The dataset contains
20 graphs for training, 2 for validation and 2 for testing. Critically, testing graphs remain completely unobserved during training. To construct the graphs, we used the preprocessed data provided
by \cite{hamilton2017inductive}. The average number of nodes per graph is 2372. Each node has 50
features that are composed of positional gene sets, motif gene sets and immunological signatures.
There are 121 labels for each node set from gene ontology, collected from the Molecular Signatures
Database (Subramanian et al., 2005), and a node can possess several labels simultaneously.

\noindent\textbf{Evaluation setup}. 
For the Reddit dataset, we use training, validation and test data split of 65\%, 10\% and 25\%, as specified in the DGLGraph library. For the other datasets, the split is 10\%, 20\% and 70\%. The same split is for evaluating the original GAT model. For training and evaluation, we closely follow the setup used in \citep{velickovic2018graph}. We first use the spectral sparsification algorithm to obtain a sparse graph and then use a two-layer GAT model for training and inference. The first layer consists of $K = 8$ attention heads, computing 8 output features each, after which we apply the exponential linear unit (ELU). The second layer consists of a single attention head that computes C features (where C is the number of classes), followed by a softmax activation. We use the same architecture while comparing with the GAT and the sparseGAT models.  We train all the models using a transductive approach wherein we use the features of all the nodes to learn the node embeddings. For the inductive learning task, we follow the evaluation method used in \cite{velickovic2018graph}. We apply a three-layer GAT model. The first two layers consist of K = 4 attention heads computing 256 features (for a total of 1024
features), followed by an ELU nonlinearity. The final layer is used for (multi-label) classification:
K = 6 attention heads computing 121 features each, that are averaged and followed by a logistic
sigmoid activation. 

\noindent\textbf{Implementation details}.
For each dataset, we compute the effective resistances of the edges using the Laplacians library written in Julia by Spielman \citep{Lap}. The rest of the algorithm is implemented in PyTorch. We use the code for the GAT provided  in \citep{velickovic2018graph}. We train our models on an Ubuntu 16.04 with 128GB memory and a Tesla P100 GPU (with 16GB memory).We use Adam optimizer with a learning rate of 0.001. We use the hyperparameters recommended in \citep{velickovic2018graph} for all of our experiments that use the GAT model. For FastGCN, we use the baseline parameters recommended in \citep{chen2018fastgcn}. 

\noindent\textbf{Computing effective resistances}.
For all datasets, computing the effective resistances is a one-time pre-processing task. We use the algorithm proposed in \citep{spielman11graph}, which takes about $ O(M \log N)$ time to compute the effective resistances of all the edges in the graph. We compute the resistance values and store them as metadata. While performing training and inference, we load the resistance values and then sample from the distribution described in Section \ref{subsec:Reff}.




\subsection{Experimental results on smaller datasets}
\label{subsec:small_data}

\begin{table}
\centering
\caption{Comparison of \mname with GAT~\citep{velickovic2018graph}.}
\label{tab:direct_comparison_small}
\resizebox{0.75 \textwidth}{!}{
\begin{tabular}{ll|c c c }
\toprule[0.6pt]
   Metric & Method & Cora & Citeseer & PubMed   \\
     \midrule
   \multirow{3}{*}{F1-micro}&  GAT & 0.72 \rpm 0.007 &0.685 \rpm 0.004 & 0.735 \rpm 0.003 \\
   ~ & \mname-0.5 & 0.713\rpm  0.003 & 0.685 \rpm 0.007 & 0.73 \rpm 0.004\\
  ~ &  \mname-0.9 & 0.63\rpm 0.02 & 0.65\rpm 0.002 & 0.722 \rpm 0.019\\
     \midrule
   \multirow{3}{*}{GPU Time (s)}&  GAT & 0.294 & 0.281 & 0.686  \\
    ~&\mname-0.5 & 0.22 & 0.236 & 0.54\\
    ~&\mname-0.9& 0.175 & 0.196 & 0.456  \\
    \midrule
     \multirow{3}{*}{CPU Time (s)}&GAT & 0.51 & 0.52 & 2.71\\
     ~& \mname-0.5 & 0.49& 0.51 & 2.01\\
    ~ & \mname-0.9 & 0.39 & 0.41&1.42\\
    \midrule
   \multirow{2}{*}{\% Edges redu.}&\mname-0.5 & 15.4\% & 10.7\% & 20\%  \\ 
   ~& \mname-0.9 & 48.7\% & 38\% & 50\%  \\
    \bottomrule[0.6pt]
\end{tabular}
}
\vskip -10pt
\label{tab:small_data_F1}
\end{table}

We report the experimental results on the smaller datasets Cora, Citeseer and Pubmed in Table \ref{tab:small_data_F1}. Since the number of edges are small compared the larger datasets, the adjacency matrices for these graphs are already considerably sparse. Hence, sparsification does not result in a large reduction in the number of edges. However, the trend is still similar to that was observed on large datasets, since the  accuracy performance does not drop, while training and inference time is lower than that for the model using the full graph.


\subsection*{Q2. \mname has the same rate of learning as GAT models}
\label{subsec:rate}

Our next goal is to study if \mname needs more epochs to achieve the same level of accuracy as that of using full graphs. Figures~\ref{fig:rate}, ~\ref{fig:rate_appendix} show consistent per epoch learning rate for multiple datasets. The accuracy achieved while training with sparsified graphs matches well with that obtained using the full graph on all the datasets, showing that spectral sparsification does not affect learning in attention GNNs. 

\begin{figure}[h!]
\vskip -10pt
    \centering
    \includegraphics[width=0.4\textwidth]{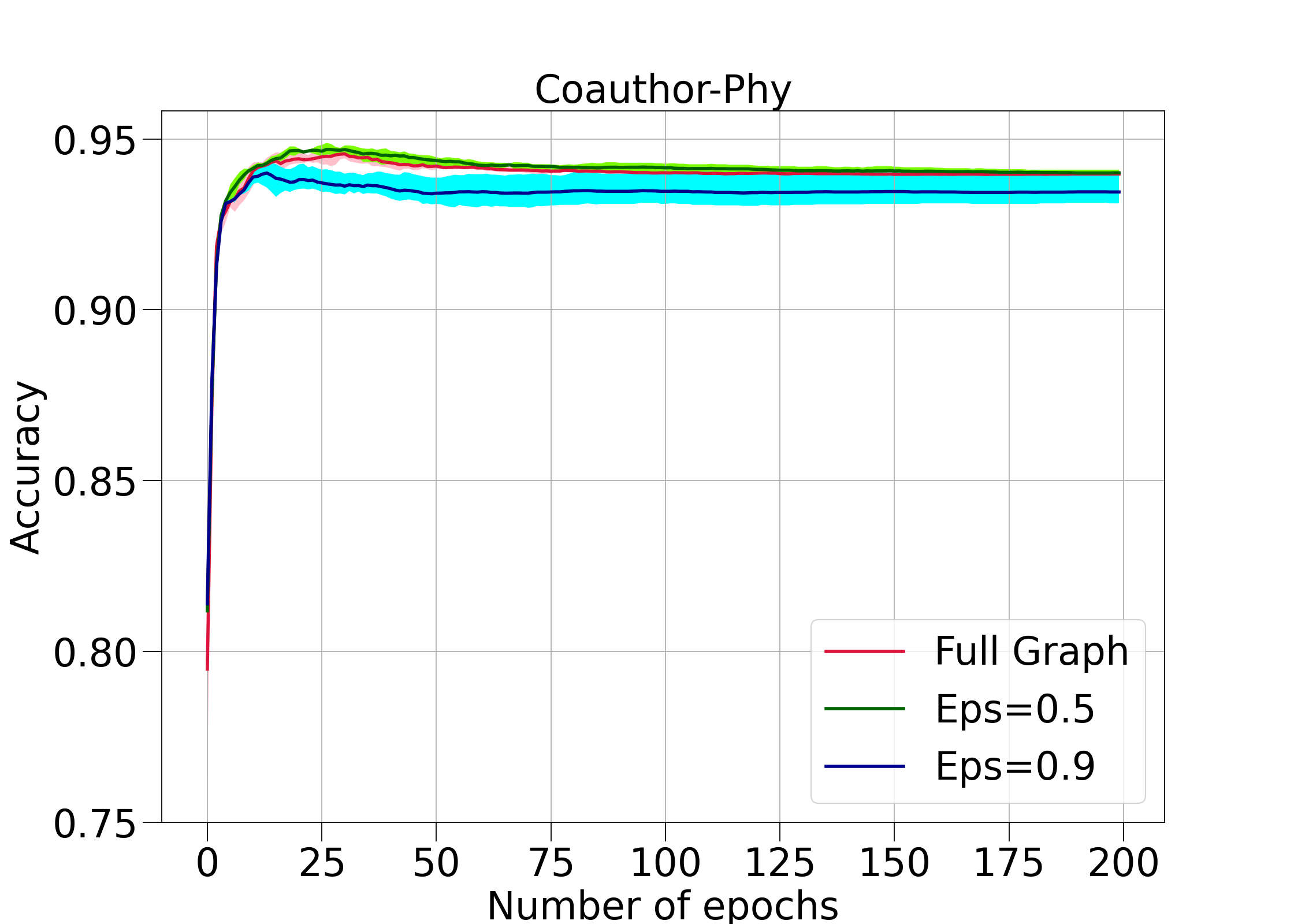}
    \includegraphics[width=0.4\textwidth]{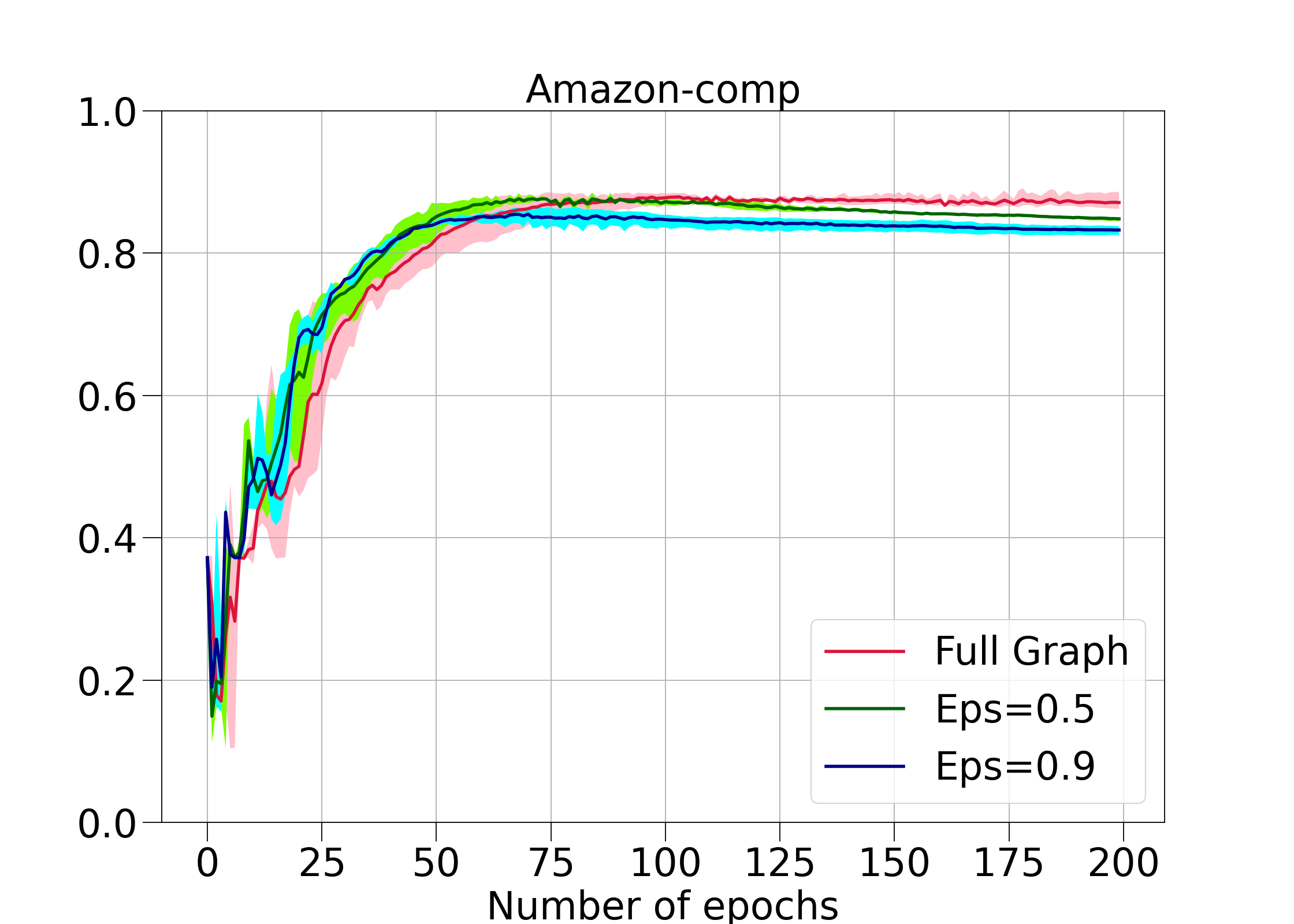}
    \caption{Accuracy Vs number of training epochs for both full and sparsified graphs. The accuracy attained matched almost exactly for every epoch even when a sparsified graph is used. For each dataset, the plots were computed by averaging over 5 independent trials. 
    }
    \label{fig:rate}
    \vskip -10pt
\end{figure}

\begin{figure}[h!]
    \centering
    \includegraphics[width=0.3\textwidth]{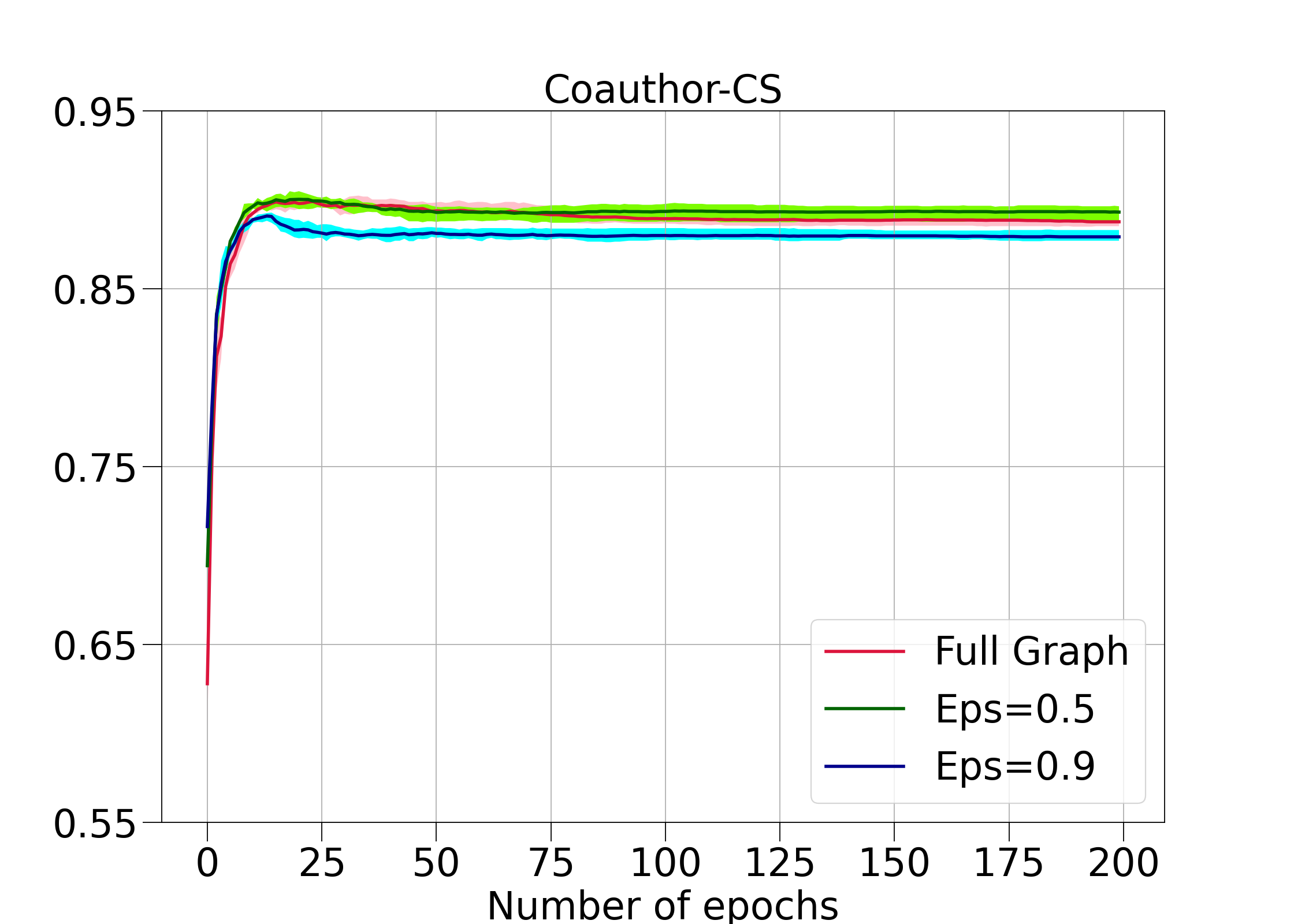}
    \includegraphics[width=0.3\textwidth]{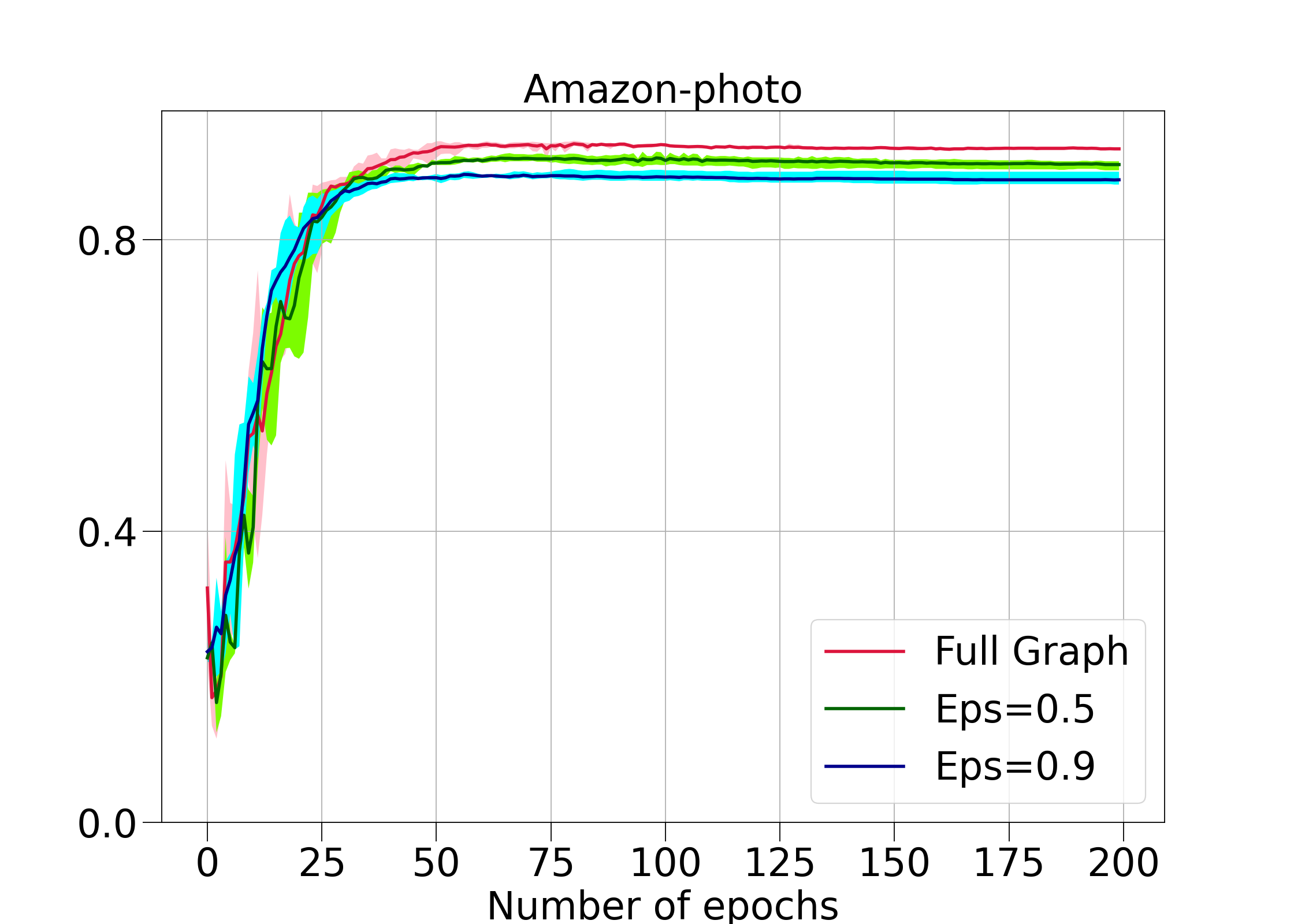}
    \includegraphics[width=0.3\textwidth]{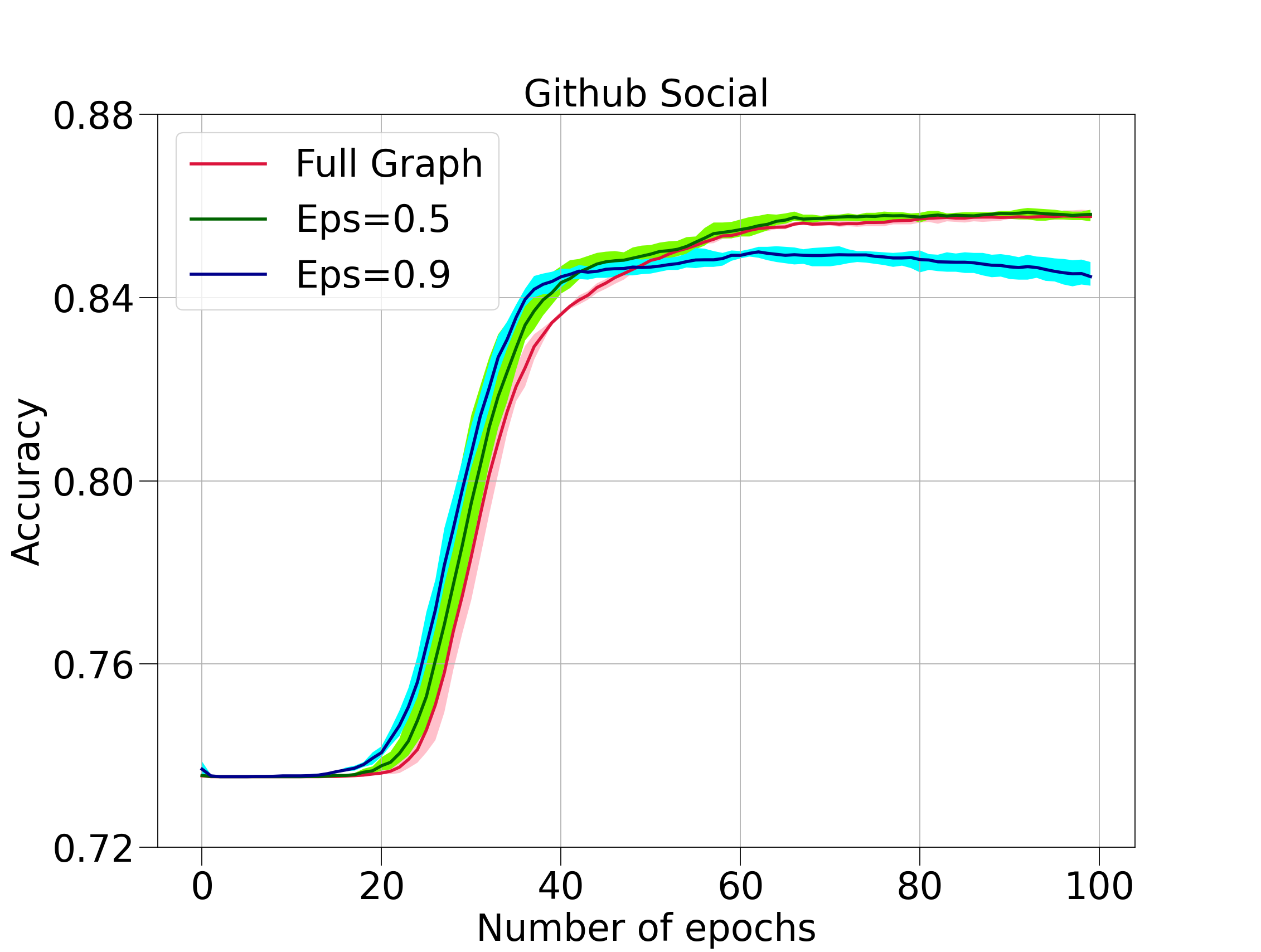}
    \caption{Accuracy Vs number of epochs for full and sparsified graphs, for the Coauthor-CS, Amazon-Photos and Github Social datasets}
    \label{fig:rate_appendix}
\end{figure}

\subsection{Fast algorithm to compute effective resistances}
\label{subsec:compute_reff}

In this section, we briefly describe the algorithm to quickly compute the effective resistances of a graph $G(\mathcal{E},\mathcal{V})$. We use the algorithm presented in \citep{spielman11graph} (Section 4) and describe it here for the sake of completion. 

For any graph $G(\mathcal{E},\mathcal{V})$, let $\bm{Y} \in \R^{M \times M}$ be such that $\bm{Y}(e,e) = w_e$ and $\bm{B} \in \R^{M \times N}$ be such that 
\begin{equation}
    \bm{B}(e,v) = \begin{cases}
    1, \text{if v is e's head} \\
    -1, \text{if v is e's tail.} \\
    0, \text{otherwise}
    \end{cases}
\end{equation} Then, it can be shown that 
\begin{equation}
    R_{uv} = \norm{\bm{Y}^{1/2}\bm{B} \bm{L}^\dagger(\chi_u - \chi_v) }^2
\end{equation} Note that the $R(uv)$'s are just pair-wise distances between the columns of the $M \times N $ matrix $\bm{Y}^{1/2}\bm{B} \bm{L}^\dagger$. The Johnson-Lindenstrauss Lemma can then be applied to approximately compute these distances. If $\bm{R}$ is a $t \times M$ random matrix chosen from a suitable distribution such as the Bernoulli distribution or the Gaussian random distribution, then if $t = O(N/\tau^2)$, then we have 
\begin{equation}
    (1-\tau) \norm{\bm{Y}^{1/2}\bm{B} \bm{L}^\dagger(\chi_u - \chi_v)}^2 \leq \norm{\bm{R}\bm{Y}^{1/2}\bm{B} \bm{L}^\dagger(\chi_u - \chi_v)}^2 \leq (1+\tau) \norm{\bm{Y}^{1/2}\bm{B} \bm{L}^\dagger(\chi_u - \chi_v)}^2.
\end{equation}
Finally, the effective resistances are computed by using a fast Laplacian linear system solver \citep{spielman2011spectral} applied to the rows of the matrix $\bm{R}\bm{Y}^{1/2}\bm{B}$. Each application of the fast solver takes $O(M \log(1/\delta))$ time where $\delta$ denotes the failure probability and can be set to a constant. The fast solver needs to be applied to $O(\log N)$ rows of the matrix $\bm{R}\bm{Y}^{1/2}\bm{B}$. Hence, the overall complexity of the algorithm is $O(M \log N)$. 

\begin{figure}[h!]
    \centering
    \includegraphics[scale=0.35]{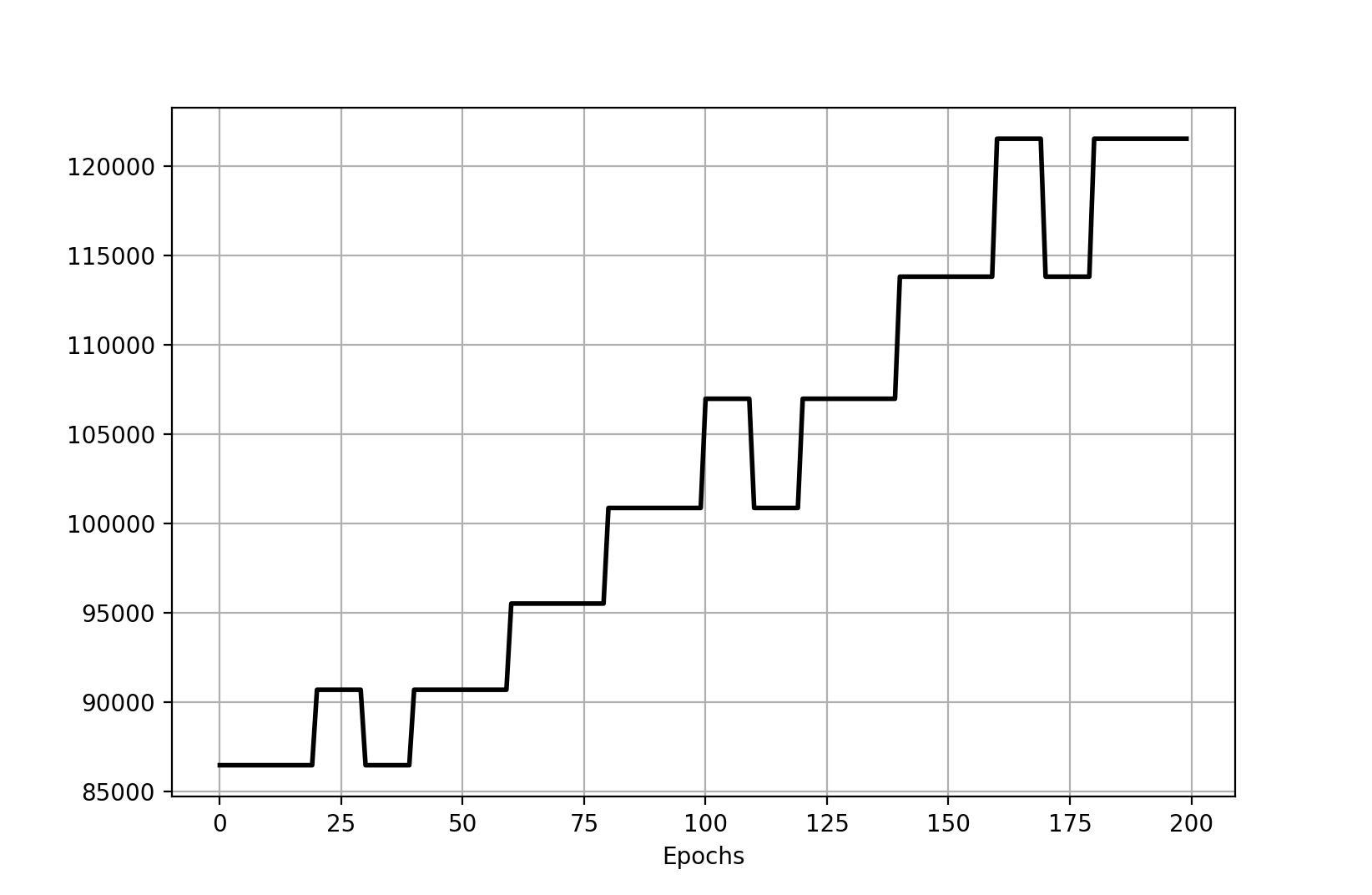}
    \caption{Number of edges selected by the adaptive algorithm for the Coauthor-Phy dataset. The number of edges at a constant epsilon of 0.5 was $163334$.}
    \label{fig:adaptive_algo_edges}
\end{figure}

\begin{figure}
    \centering
    \includegraphics[scale=0.4]{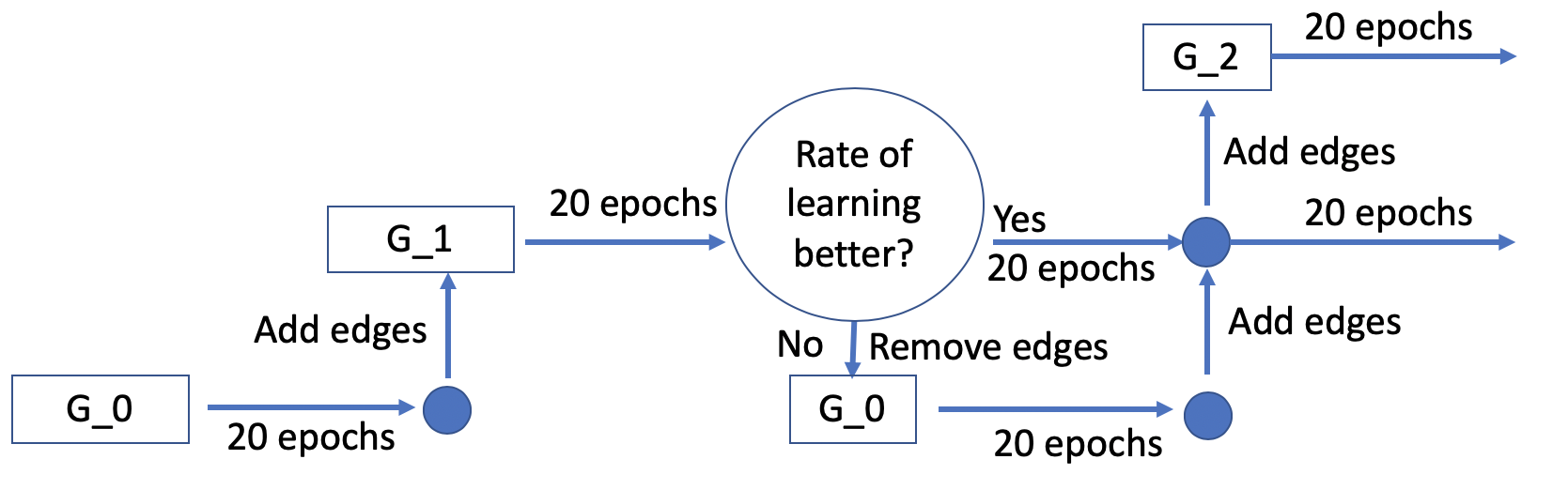}
    \caption{Adaptive algorithm to tune epsilon parameter (or the number of edges). We start with a sparse graph and iteratively build denser graphs as we progress through the epochs. In the ``Add edges" step, we add a fixed number ($0.003M$) of edges to the graph. In the ``Rate of learning better?" step, we compare the slopes of the training accuracy curve with the previous slope over 20 epochs. } 
    \label{fig:adaptive_algo}
\end{figure}

\begin{figure}
    \centering
    \includegraphics[scale=0.25]{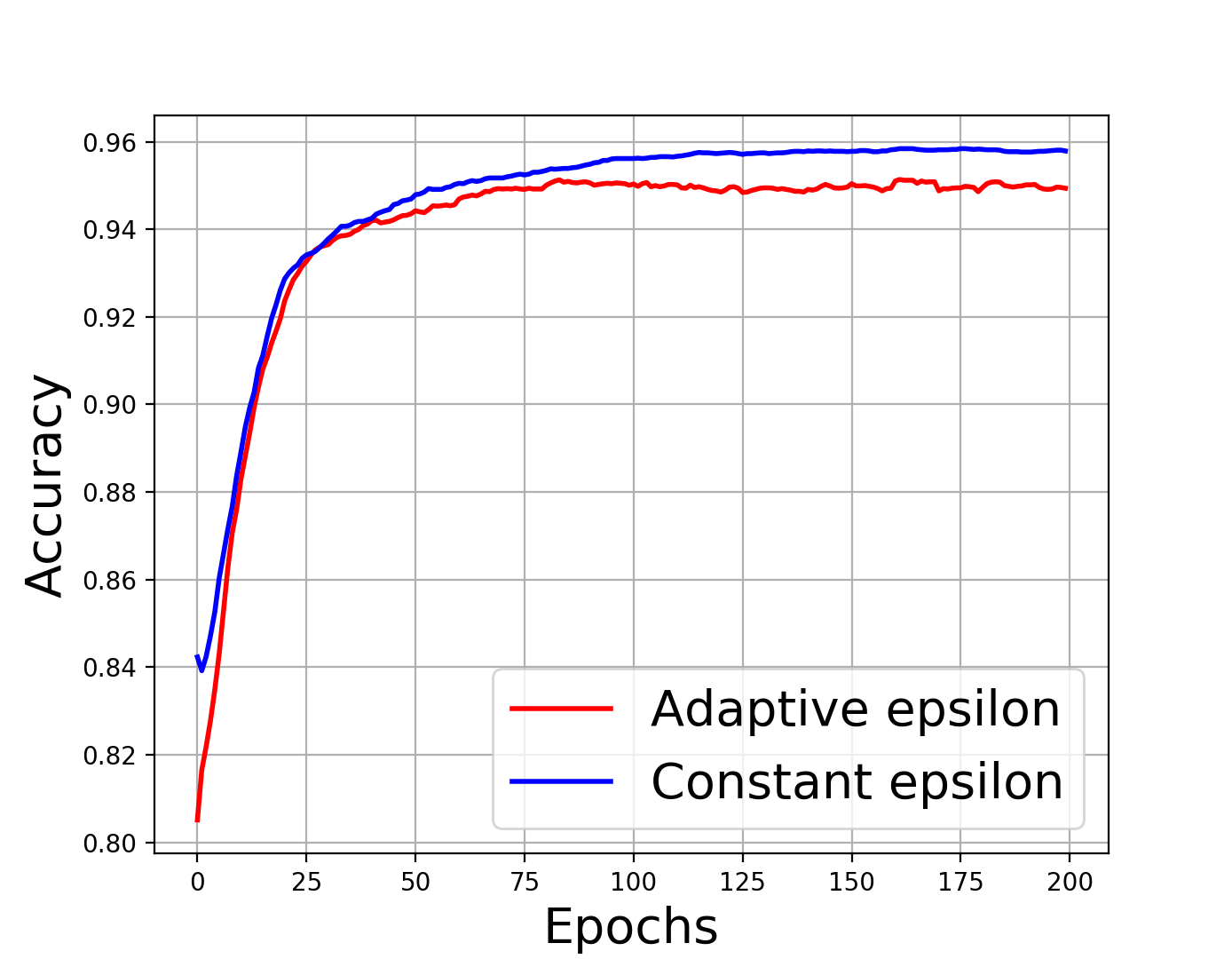}
    \includegraphics[scale=0.25]{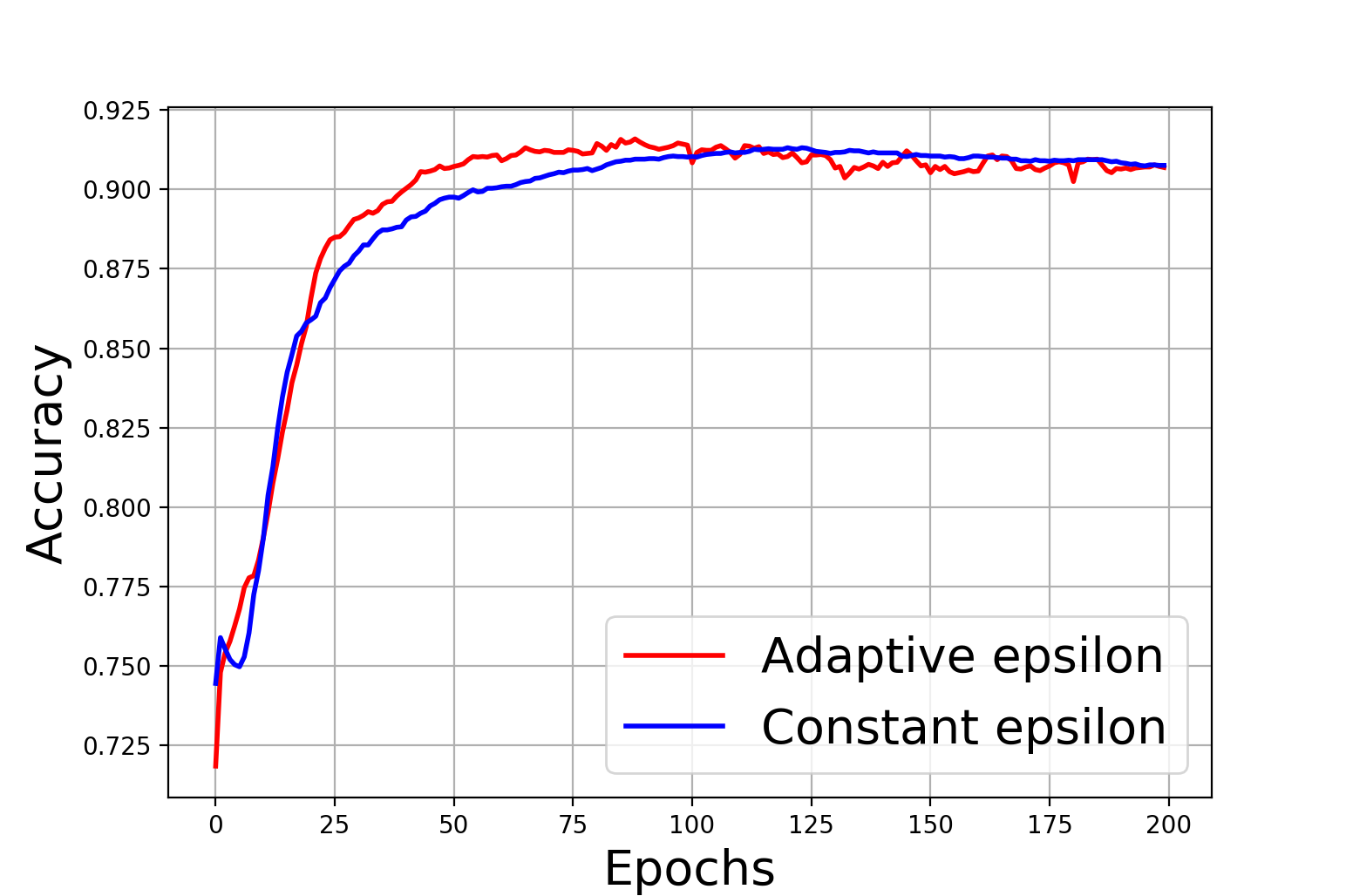}
    \includegraphics[scale=0.25]{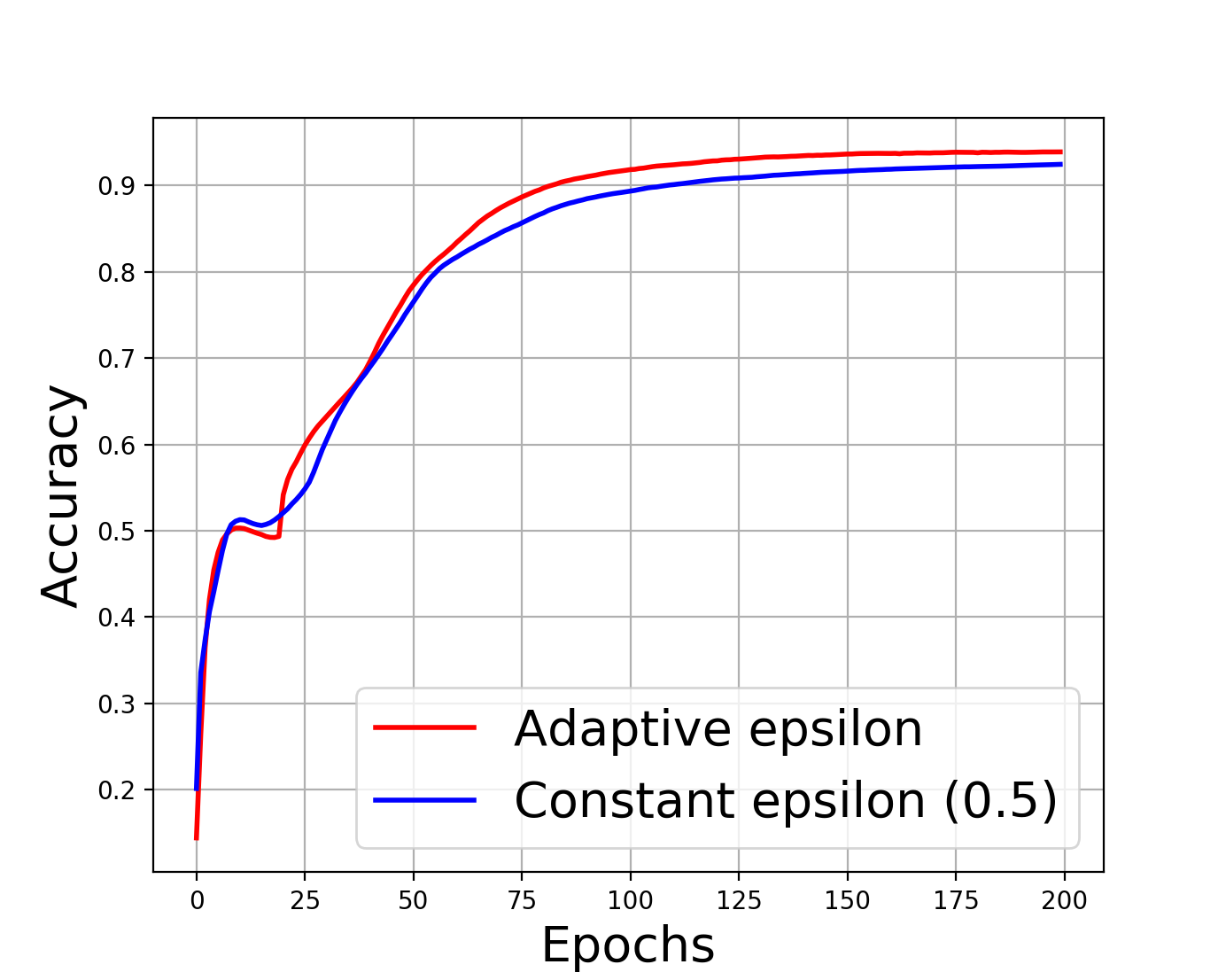}
    \caption{Simulation results for the adaptive algorithm on the coauthor-Physics, coauthor-CS and Reddit datasets.}
    \label{fig:adaptive_algo_acc}
\end{figure}

\subsection{Adaptive sparsification algorithm}
\label{subsec: epsilon_effect}

In the previous sections, we showed that for a suitable value of the tolerance parameter $\epsilon$ (such as $0.5$, $0.9$), the accuracy is equivalent to that of using the full graph. However, the level of sparsification needed to maintain the classification performance might be different for different datasets. This raises a very natural question of how to design the $\epsilon$ parameter for different datasets. In this subsection, we seek to address this question.

We provide here an algorithm that sweeps through various values of $\epsilon$ and achieves state of the are results on any given dataset. In our experience, we find that using $\epsilon=0.5$ produces test accuracies that are as good as that of using the full graph. Hence, we set $0.5$ as the minimum value of $\epsilon$ that our algorithm chooses. It iteratively chooses a denser or a sparser graph based on the current validation error of the algorithm. We provide a block diagram of the algorithm in Figure \ref{fig:adaptive_algo}.  In Figure. \ref{fig:adaptive_algo_acc}, we show the training accuracy Vs the epochs for our algorithm and compare it with that of a model using a constant $\epsilon$ of $0.5$. From the figure, it is evident that our adaptive algorithm is successful in achieving the same learning rate as that of a model with constant $\epsilon$. Hence this algorithm is suitable to be deployed as is on other real world datasets. 

In Figure. \ref{fig:adaptive_algo_edges}, we show the the number of edges resulting edges in the graph after each instance of the algorithm choosing to sparsify or make the graph more dense. Since denser graphs do offer more information, it is natural that the algorithm chooses denser graphs over time in general. But it is also interesting to see that there are instances where the algorithm chooses a sparser graph. We show the accompanying time per epoch as well in Figure.  where we can see that it is much smaller than that of using a constant, low $\epsilon$ parameter.



\end{document}